\documentclass[sigconf,preprint]{acmart}
\usepackage{booktabs}
\usepackage{arydshln}
\usepackage{makecell}
\usepackage{enumitem}
\usepackage{graphicx}
\usepackage{soul}
\usepackage{url}
\usepackage{rotating}
\usepackage{amsmath}
\usepackage{amsthm}
\usepackage{algorithm}
\usepackage{algorithmic}
\usepackage{multirow}
\usepackage{marvosym}
\usepackage{threeparttable}
\usepackage{mathrsfs}
\usepackage{color,xcolor}
\usepackage{colortbl}
\AtBeginDocument{%
  }
\setcopyright{none}
\begin{document}

\title[Emergent Modularity]{Learning Emergent Modular Representations in Multi-modality Medical Vision Foundation Models}

\author{Yuting He}
\email{yuting.he4@case.edu}
\affiliation{%
  \institution{Case Western Reserve University}
  \city{Cleveland}
  \state{Ohio}
  \country{USA}
}

\author{Chenyu You}
\email{chenyu.you@stonybrook.edu}
\affiliation{%
  \institution{Stony Brook University}
  \city{Stony Brook}
  \state{New York}
  \country{USA}
}

\author{Shuo Li}
\authornote{Corresponding Author.}
\email{shuo.li11@case.edu}
\affiliation{%
  \institution{Case Western Reserve University}
  \city{Cleveland}
  \state{Ohio}
  \country{USA}
}

\begin{abstract}
Multi-modality medical vision (MV) foundation models (FM) are fundamentally challenged by pronounced \textit{Non-IID} feature statistics across heterogeneous imaging modalities. Monolithic self-supervised optimization on such data induces conflicting gradients, driving representations to collapse toward modality-dominant shortcuts. This work reframes this failure as an imbalance between specialization and coordination in \textit{emergent modularity}, and proposes \textit{Director–Experts} (\textbf{DEX}), a modular network that explicitly regulates these dynamics in stacked modules. Each DEX module comprises a pool of \textit{experts}, dynamically adapted by our image-wise activation strategy, autonomously specializing in modality-dominant statistics, together with a \textit{director}, updated via our group exponential moving average, which distills multi-expert knowledge into a shared space for semantic integration across modalities, thus driving the emergence of modular representations. We curate a new benchmark, Medical Vision Universe, over 4 million images across 10 modalities, which provides a FM–level pre-training with the broadest coverage of distinct imaging modalities to our DEX. Extensive evaluations on 26 downstream tasks demonstrate improved optimization behavior and transferability, indicating DEX as a principled step toward general-purpose multi-modality medical AI. Our code and dataset will be opened at \href{https://github.com/YutingHe-list/DEX}{https://github.com/YutingHe-list/DEX}.
\end{abstract}

\maketitle
\section{Introduction}
\label{sec:intro}
Multi-modality medical vision (MV) foundation models (FMs) promise unified, general-purpose representations that are able to be adapted across a wide range of clinical imaging tasks \cite{ma2024segment,ren2024medical,he2024foundation,bommasani2021opportunities}. However, a fundamental obstacle to this goal is the pronounced \emph{Non-Independent and Identically Distributed (\textbf{Non-IID})} nature of feature statistics across medical imaging modalities. Unlike natural images that share common optics and statistical priors, medical modalities (CT, MR, Ultrasound, etc.) are governed by distinct imaging physics and acquisition protocols, producing heterogeneous signals with disjoint feature distributions that violate the IID assumptions underlying large-scale self-supervised recipes \cite{bischof2024multimodal,li2024multimodal,cao2022beyond}. When trained with monolithic self-supervised learning objectives, which is standard in uni-modality success \cite{huang2023visual,lu2024visual} (Fig.\ref{fig:challenge}-a), these heterogeneous signals induce conflicting gradients during back-propagation. This optimization conflict will interfere the pretraining, forcing the network to collapse toward modality-dominant shortcuts, e.g., distinguishing X-ray from MRI based on noise patterns rather than semantics \cite{geirhos2020shortcut,javaloy2022mitigating}. The practical consequence is representation collapse that pretrained features become dominated by modality cues and limit the adaptation across medical scenarios (Fig.~\ref{fig:challenge}-b).

\begin{figure}
  \centering
  \includegraphics[width=\linewidth]{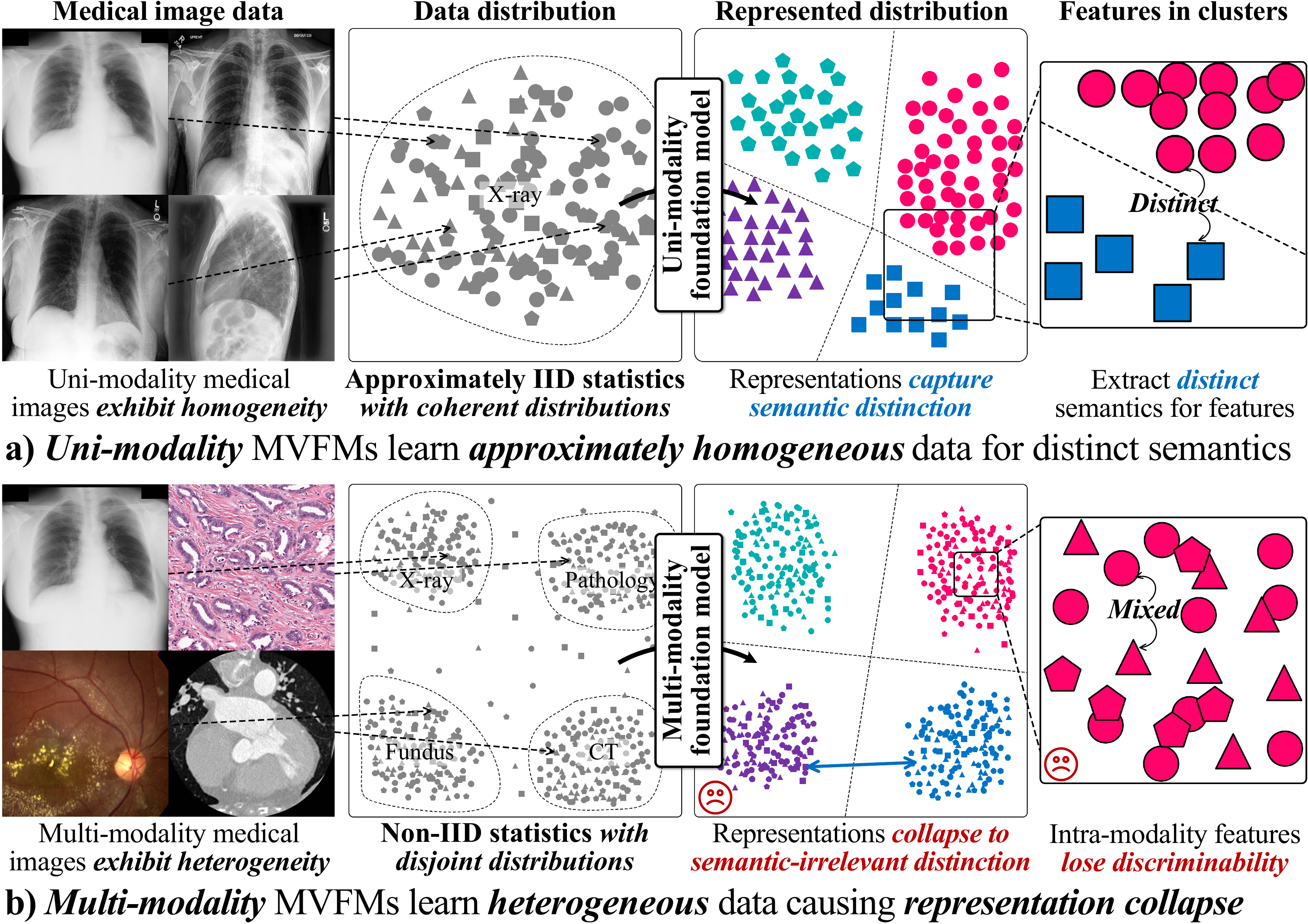}
  \caption{Heterogeneity in multi-modality MV data causes the Non-IID problem. a. Single modality data is approximately IID, enabling coherent representation learning. b. In multi-modality settings, pronounced heterogeneity produces disjoint feature distributions with Non-IID statistics. Monolithic training on such Non-IID data leads to gradient interference, causing modality-dominated representation collapse.}
  \label{fig:challenge}
\end{figure}

Mechanistically, the representation collapse in multi-modality MVFMs arises from the interference of Non-IID data statistics with the network's \textbf{emergent modularity} \cite{qiu2024unlocking,csordasneural,zhang2024inductive,zhang2023emergent,schaeffer2023emergent,pfeiffer2023modular,fang2025emoe}. Deep neural networks naturally tend to self-organize parameter subspaces into functionally specialized modules during optimization \cite{qiu2024unlocking,csordasneural}, a process that ideally disentangles complex data factors. However, in our multi-modality MVFMs, this self-organization is severely disrupted by the pronounced heterogeneity across imaging modalities. When trained under a monolithic objective, the conflicting gradients from disjoint feature distributions act as stochastic interference, preventing the formation of stable functional modules. Instead of settling into specialized modules, parameter subspaces are trapped in a stochastic tug-of-war driven by incompatible modality signals \cite{zhang2024inductive}. Consequently, the potential for modular specialization is suppressed, and the representation degenerates into modality-dominant shortcuts (Fig.~\ref{fig:challenge}-b). Overcoming this collapse requires regulating the optimization dynamics to shield the emergent modularity from such heterogeneity-induced interference \cite{pfeiffer2023modular,fang2025emoe}.

\begin{figure}
  \centering
  \includegraphics[width=\linewidth]{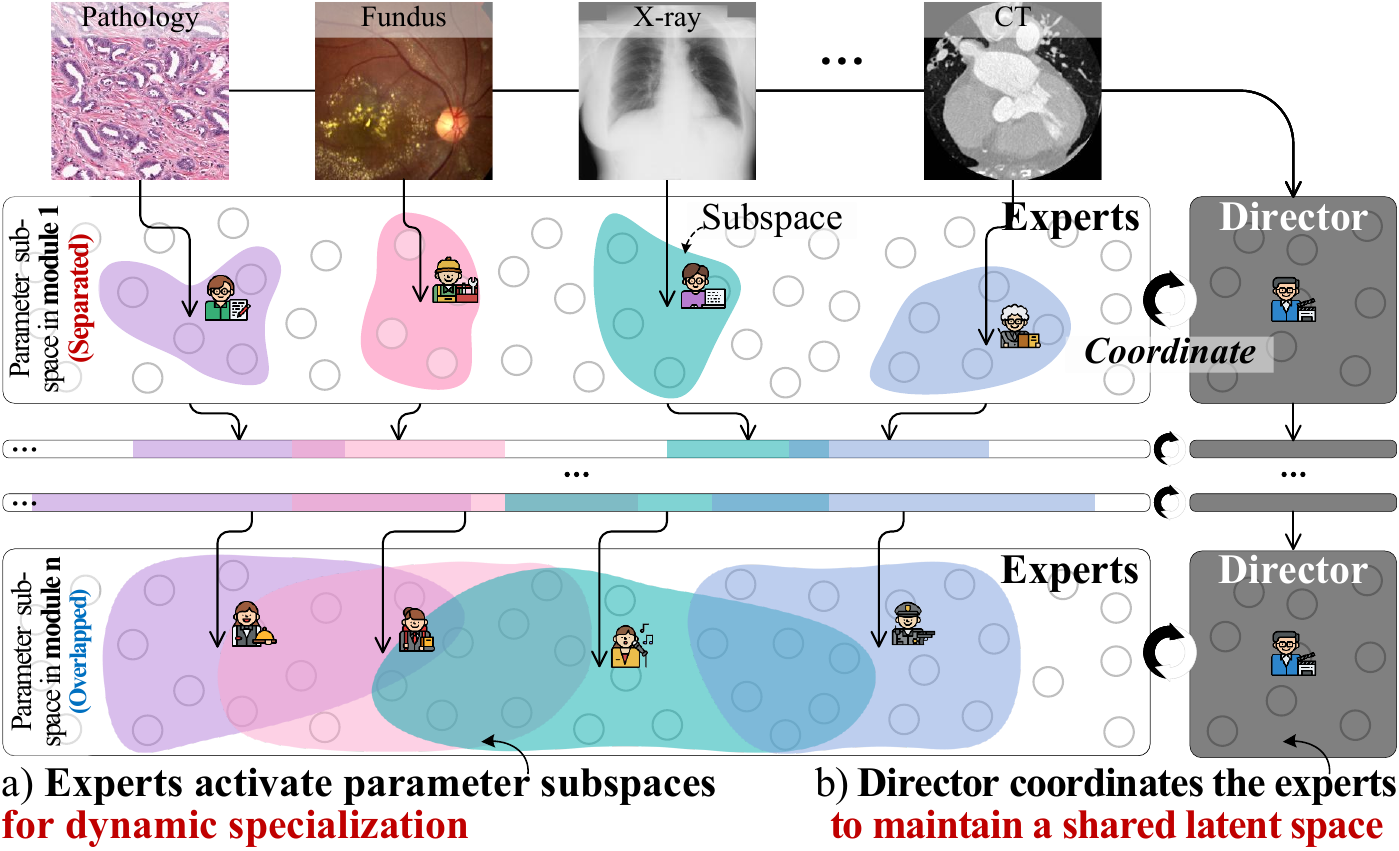}
  \caption{Our Director-Experts (DEX) modular network learns heterogeneous multi-modality MVFMs by explicitly regulating specialization and coordination. a. Multiple experts dynamically activate parameter subspaces to specialize modality-dominant statistics. b. A director coordinates the experts to learn a shared space to align knowledge. Hierarchical stacking of these DEX modules enables integration for the emergence of modular representations.}
  \label{fig:modeling}
\end{figure}

Unlocking this emergent modularity requires managing two complementary yet competing optimization dynamics, i.e., \textit{specialization} and \textit{coordination} \cite{csordasneural,zhang2024inductive}. \textbf{a. Specialization} directs parameter subspaces to focus on modality-dominant statistics, splitting conflicting updates and reducing interference. However, uncontrolled specialization will push subspaces toward isolated local optima, fragmenting representations and hindering semantic transfer across modalities \cite{csordasneural}. \textbf{b. Coordination} enforces models to learn shared feature structure among specialized subspaces, promoting information exchange and alignment, and enabling perceptual cues to progressively integrate into semantic features. However, excessive coordination will suppress necessary diversity, preventing subspaces from capturing modality-specific signals \cite{singh2022flava,radford2021learning,wang2021knowledge}. \textbf{\textit{Current Gap}:} Existing paradigms emphasize one dynamic at the expense of the other, following independent technical paths and lacking principled integration. For example, although Mixture-of-Experts (MoE) \cite{chopra2025medmoe,li2025uni} improves specialization, the its feature-/token-level routing misassigns local cues that are similar across modalities yet semantically distinct, fragmenting holistic imaging patterns and failing to achieve the modality-level specialization needed to decouple distinct imaging physics. In contrast, although self-distillation \cite{yang2024vitkd,li2024multimodal,wang2021knowledge} enhances coordination, its single-teacher implementations are unable to reconcile knowledge across multiple specialized subspaces into a shared representation. These paradigms have historically been pursued in isolation \cite{li2024multimodal,wang2021knowledge}, without a mechanism to jointly regulate specialization and coordination. As a result, under the extreme Non-IID condition in our multi-modality MVFMs, these isolated solutions fail, leaving this tension an open challenge.

In this paper, we propose the \textit{Director-Experts} (\textbf{DEX}) modular networks, which explicitly regulate the specialization-coordination dynamics to learn the emergent modular representation in multi-modality MVFMs under Non-IID conditions (Fig.~\ref{fig:modeling}). DEX reformulates the monolithic backbone into modular groups (termed DEX modules), decomposing learning into multiple parameter subspaces and their interactions via two key components. \textbf{a. Experts for autonomous specialization.} Each DEX module contains a pool of experts whose combinations are activated to form adaptive parameter subspaces. By routing inputs to expert combinations, DEX isolates conflicting gradients and lets subspaces absorb modality-dominant statistics. To implement modality-level routing, we design an image-wise expert activation strategy which enables routing entire images to the experts that match their global features relevant to modality, avoiding the misassignment of local cues in feature-/token-level routing \cite{chopra2025medmoe}. This acts as a modality buffer, mitigating destructive interference while remaining efficient for high-resolution medical data.
\textbf{b. Director for dynamic coordination.} To cope with the fragmentation from isolated specialization, our director module serves as a global semantic anchor across modalities. It is updated via our proposed group exponential moving average (GEMA), which dynamically tracks expert activation patterns and distills multi-expert knowledge into a shared representation space, avoiding the reliance on static averaging or single-teacher distillation \cite{yang2024vitkd,li2024multimodal,wang2021knowledge}. By imposing an alignment constraint that forces divergent experts to converge toward this shared space, the director promotes semantic integration across modalities. Stacking DEX modules hierarchically implements progressive representation evolution: shallow modules emphasize expert specialization to capture low-level perceptual features, while deeper modules strengthen director-mediated coordination to facilitate inter-modality communication and shared abstractions. Collectively, DEX balances specialization and coordination across the representation hierarchy, enabling the emergence of modular representations.

We pretrained our DEX on a novel large-scale multi-modality MV benchmark, Medical Vision Universe (MedVerse), and evaluated it with 26 downstream tasks, whose superior performance illustrates our great potential. Our MedVerse comprises over 4 million 2D (following \cite{mh2023lvm}) medical images spanning 10 diverse imaging modalities (see \textit{Appendix} for details). Under the clinical-modality definition, it provides the broadest coverage of distinct imaging modalities among publicly available 2D datasets for FM–level self-supervised pretraining. At a practical pretraining scale, DEX exhibits improved optimization dynamics and strong transferability. Extensive evaluations on 26 downstream tasks across multiple modalities demonstrate its powerful generalizability, indicating that DEX provides a modular representation foundation and represents an important enabling step toward general-purpose multi-modality medical AI. 

To the best of our knowledge, this is the first study on the problem of the Non-IID condition in multi-modality MVFMs, and proposes the DEX modular network that unifies heterogeneous modalities within an FM. Our contributions are summarized as follows: 1) We propose the DEX modular network, which regulates the optimization dynamics of specialization and coordination in our proposed experts and directer modules, thus enabling the learning of emergent modular representations for heterogeneous multi-modality medical images. 2) We propose the image-wise expert activation in our expert modules, enabling the selection of parameter subspaces matching images' global features. 3) We propose the group exponential moving average in our director module, enabling the distillation of multi-expert knowledge into a shared representation space. 4) We curate the MedVerse, a large-scale multi-modality MV dataset with over 4M 2D images spanning 10 modalities, enabling FM–level pretraining for diverse clinical tasks. 5) We construct a novel MVFM based on our DEX and MedVerse, improving the infrastructure to advance the MV research. Extensive experiments across 26 downstream tasks demonstrate our strong transferability, showing the advantages of emergent modular representations.

\section{Related Works}
\textbf{1) Medical vision foundation models (MVFMs).} MVFMs \cite{huang2023visual,lu2024visual,chen2024towards,wang2024pathology} leverage large-scale self-supervised pretraining to learn transferable representations for diverse downstream tasks, enhancing generalization and accelerating the development of medical AI \cite{he2024foundation}. Most existing models remain uni-modality (e.g., X-ray \cite{ma2025fully}, pathology \cite{huang2023visual}), capturing modality-specific semantics but struggling to generalize beyond pretraining distribution. This uni-modality constraint limits their utility in real clinical workflows that require integrated multi-modality data for decision-making. Recent multi-modality extensions \cite{ma2024segment,ren2024medical} show promising cross-modality capabilities, but rely heavily on curated paired data and remain task-specific rather than learning common representations.

\textbf{2) Multi-modality medical vision learning.} It integrates heterogeneous modalities (e.g., CT, X-ray) to achieve more comprehensive clinical understanding and superior predictive performance over uni-modality models \cite{taleb2021multimodal,warner2024multimodal,bayoudh2022survey}. Existing works fall into three categories: \textbf{a.} \textit{Image-level learning} fuses raw multi-modality images for joint modeling \cite{safari2023medfusiongan,li2025bsafusion}, but suffers from resolution and alignment inconsistencies. \textbf{b.} \textit{Feature-level learning} adopts modality-specific encoders and interaction modules, e.g., concatenation or attention, to mitigate heterogeneity and enhance cross-modality fusion \cite{warner2024multimodal,tang2022matr,feng2022multimodal}. However, such designs are task-specific, limiting their generalization to unseen modality combinations. \textbf{c.} \textit{Representation-level learning} unifies heterogeneous modalities into a shared representation space through large-scale pretraining \cite{bischof2024multimodal,he2024foundation}, enabling transferable representations. Current efforts rely on paired or annotated data \cite{chen2025mimo,Xue_2023_CVPR,ma2024segment,ren2024medical}, struggling under self-supervised objectives due to the Non-IID condition \cite{cao2022beyond}.

\textbf{3) Non-IID conditions in multi-modality medical vision.} The Non-IID conditions arise naturally in multi-modality MV, where diverse imaging modalities are generated under different imaging physics, spatial resolution, scale, and acquisition protocols \cite{cao2022beyond,bischof2024multimodal}. This results in disjoint feature distributions, violating the IID assumption commonly relied upon in FM pretraining paradigms \cite{he2024foundation,cao2022beyond}. Treating multi-modality data as IID risks overfitting to the modality-dominant distinctions, rather than semantically meaningful structures internal modalities \cite{chaudhuricloser}. Attempts to mitigate this challenge include domain adaptation \cite{guan2021domain}, distribution calibration \cite{verma2025calibration}, and federated learning \cite{zhu2021federated}, but these methods assume distribution overlap or label availability limited in MVFM practice. Therefore, the Non-IID conditions remain an open challenge.

\textbf{4) Emergent modularity and modular networks.} Emergent modularity refers to the phenomenon where neural networks spontaneously organize parameters into functionally specialized subspaces during training, without explicit constraints \cite{csordasneural,zhang2023emergent,pfeiffer2023modular}. This behavior has been observed across language models \cite{qiu2024unlocking,wang2025modular,schaeffer2023emergent}, spatial navigation \cite{zhang2024inductive}, and multi-modality tasks \cite{fang2025emoe}, where submodules specialize in capturing particular features or task components. Modular networks explicitly structure these subspaces to regulate specialization and inter-module coordination, which can be grouped into several categories: \textbf{a.} \textit{MoE architectures} \cite{chopra2025medmoe,li2025uni,6796382,NEURIPS2024_4a3a14b9} partition parameters into multiple experts to enable specialization; \textbf{b.} \textit{routing or shared-unit architectures} \cite{wang2021codinet,misra2016cross} adaptively route inputs or share subspaces across tasks or modalities. While these works capture partial benefits of emergent modularity, they regulate independent specialization or coordination, and control across hierarchical layers remains largely unexplored. 

\textbf{5) Optimization dynamics of specialization and coordination.} Two complementary dynamics internal networks have been explored for the Non-IID conditions in multi-modality MV. \textbf{a.} \textit{Specialization} separates modality-specific information into dedicated modules (e.g., MoE \cite{chopra2025medmoe,li2025uni}, dedicated encoders \cite{zong2024self,li2024multimodal}, adaptive routing \cite{wang2021codinet}), preserving modality-specific semantics but fragmenting shared knowledge. \textbf{b.} \textit{Coordination} enforces a shared latent space via contrastive learning \cite{radford2021learning,singh2022flava}, cross-attention \cite{shi2022xmorpher}, or optimal transport \cite{chenplot}, enabling inter-modality interaction. However, coordination is typically applied at fixed representation levels \cite{almudevaraligning} or single-teacher implementations, limiting flexibility across semantic granularities and multi-module settings. 

\section{Director-Experts Modular Network}
\begin{figure*}
  \centering
  \includegraphics[width=0.84\linewidth]{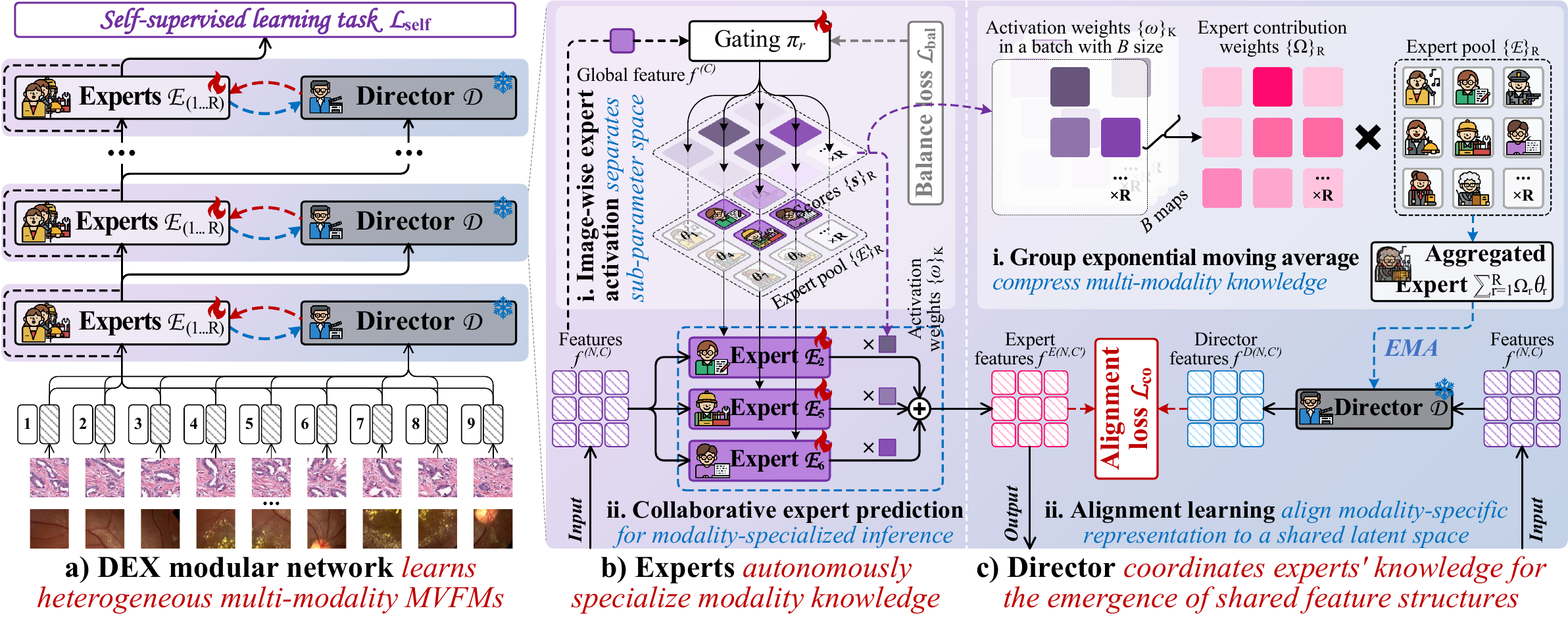}
  \caption{\textbf{Framework:} DEX modular networks regulate heterogeneous multi-modality MV representations within networks. a. The DEX network stacks hierarchical modules to enable specialization and coordination, progressively integrating modality-specific features into shared abstractions. b–c. Each DEX module comprises multiple experts and one director, which autonomously specialize parameter subspaces and coordinate them for the emergence of modular representations.}
  \label{fig:framework}
\end{figure*}
Our DEX (Fig.\ref{fig:framework}) reformulates the backbone as a modular network (Sec.\ref{subsec:network}), enabling the learning of emergent modular representation in multiple modular groups (our DEX module). It combines multiple expert modules (Sec.\ref{subsec:actors}) and a director module (Sec.\ref{subsec:director}) to balance specialization-coordination dynamics, enabling the self-organization of multi-modality knowledge.

\subsection{Problem Formulation}

The proposed DEX reformulates an FM network as a modular network where optimization dynamics are organized to unify multi-modality representations. Formally, the entire network $\mathcal{F}$ consists of $L$ stacked modules $\{\mathcal{M}_1,\dots,\mathcal{M}_L\}$, where each $\mathcal{M}_l$ module performs two dynamics: specialization and coordination. Given input features $f$, the representation in each module is expressed as
\begin{equation}
    \begin{matrix}\hat{f}=\Phi_{\text{sp}}(f), \quad\Phi_{\text{co}}(\hat{f},\Theta)\end{matrix},
\end{equation}
where $\Phi_{\text{sp}}(\cdot)$ denotes a specialization dynamic that captures modality-specific features $\hat{f}$ within parameter subspaces, and $\Phi_{\text{co}}(\cdot,\cdot)$ represents a coordination dynamic that aligns these specialized features to a shared space $\Theta$ to learn the shared feature structure. 

Our proposed DEX comprises multiple \textit{experts} $\mathcal{E}$ with a gating function to achieve the specialization $\Phi_{\text{sp}}$ and one \textit{director} $\mathcal{D}$ to model a shared latent space $\Theta$ thus driving the coordination $\Phi_{\text{co}}(\cdot,\cdot)$. Given an input $f$ from a preceding layer,
\begin{equation}
    \begin{matrix}\hat{f}^{E}=\Phi_{\text{sp}}(f)=\sum_{r=1}^{R}\pi_{r}(f)\mathcal{E}_{r}(f)\end{matrix},
\end{equation}
where $\pi_{r}\in\{\pi\}_{R}$ is a learnable gating function controlling the activation of experts that specialize in specific modalities like MoE \cite{chopra2025medmoe}. The Export output $\hat{f}^{E}$ captures modality-specific features as the module output, achieving the specialization. To further regulate the coordination, the director projects the $f$ to a shared latent space for $\hat{f}^{D}=\mathcal{D}(f)$, and further aligns the $\hat{f}^{E}$ to the $\hat{f}^{D}$, thus driving the specialized features to a shared space, i.e.,
\begin{equation}
    \begin{matrix}\quad\Phi_{\text{co}}(\hat{f}^{E},\Theta^{D})=\min_{\theta}\mathbb{E}[\mathcal{L}_{\text{co}}(\hat{f}^{E},\hat{f}^{D})]\end{matrix},
\end{equation}
where the $\theta$ is the experts' activated parameter subspace, and the $\mathcal{L}_{\text{co}}$ is the alignment loss function aligning features.

The DEX jointly optimizes specialization and coordination by minimizing a self-supervised loss $\mathcal{L}_{\text{self}}$ as a global objective for representation learning in multiple experts, and the alignment loss $\mathcal{L}_{\text{co}}$ to coordinate expert outputs with the director \cite{colson2007overview}:
\begin{align}\label{equ:bi}
\begin{matrix}\min_{\theta}\end{matrix}\quad
&\begin{matrix}\mathbb{E}[\mathcal{L}_{\text{self}}(\hat{f}^{E}(\theta))+\mathcal{L}_{\text{co}}(\hat{f}^{E}(\theta),\hat{f}^{D}(\eta^*(\theta)))],\end{matrix}\notag\\
\text{s.t.}\quad 
& \begin{matrix}\eta^*(\theta)=\mathcal{G}_{\text{EMA}}(\theta,\eta),\end{matrix}
\end{align}
where $\theta$ and $\eta$ denote experts and director parameters, and $\mathcal{G}_{\text{EMA}}$ is our designed group exponential moving average (GEMA) algorithm that updates the director with knowledge aggregated from the experts for a shared latent space $\Theta$. This formulation enables the experts and director to co-evolve: the experts specialize through self-supervised and alignment learning, while the director consolidates their knowledge. In practice, the optimization dynamics in Equ.\ref{equ:bi} are efficiently executed in an alternating update process \cite{xiao2023alternating}, where the experts minimize the outer objective while the director is updated accordingly, enabling a joint optimization that self-organizes specialization–coordination dynamics and facilitates the emergence of modular representations. 

\subsection{Experts \textbf{\textit{specialize modality knowledge}}}\label{subsec:actors}
The experts $\mathcal{E}$ in each DEX module $\mathcal{M}$ autonomously activate parameter subspaces, specializing modality knowledge and avoiding uncontrollable interference from others (Fig.~\ref{fig:framework}-b). Specifically,

\textbf{Image-wise expert activation} Unlike prior works \cite{fedus2022switch}, it activates the experts based on the holistic modality information of the entire image rather than individual tokens, distributing different modality images to their specialized experts and reducing the computation in activation. Given a feature map with $N$ features and C dims $f^{(N,C)}$ from the preceding layer, the $l$-th DEX module $\mathcal{M}_{l}$ extracts a global feature $f^{(C)}$ via the class token \cite{dosovitskiy2020image} or global average pooling to represent the whole image. This global feature is then projected through a learnable activation matrix $\pi^{(C,R)}$ to produce activation scores $s^{(R)}=\text{softmax}(f^{(C)}\times \pi^{(C,R)})$ where $R$ is the number of experts in the expert pool. Experts with the top-$K$ highest scores are activated to predict the expert features. Following \cite{fedus2022switch}, we added a balance loss $\mathcal{L}_{\text{bal}}$ to avoid the activation collapse.

\textbf{Collaborative expert prediction} The $K$ activated experts collaboratively define a sub-parameter space, enabling specialized prediction over the feature map $f^{(N,C)}$. Each activated expert $\mathcal{E}_{k}$ processes the $f^{(N,C)}$ to capture the specific properties of the input for $\hat{f}^{E(N,C')}_{k}$ ($C'$ is the new dim). The predicted features from all $K$ activated experts are aggregated through a weighted summation based on the activation scores, yielding the output of the experts block for the overall expert features $\hat{f}^{A(N,C')}$, i.e.,
\begin{equation}
\hat{f}^{E(N,C')} = \begin{matrix}\sum_{k=1}^{K}\omega_{k}\mathcal{E}_{k}(f^{(N,C)})\end{matrix},
\label{eq:actor_prediction}
\end{equation}
where $\omega_{k}=\frac{s_{k}}{\sum_{k'=1}^{K} s_{k'}}\in\{\omega\}_K$ denotes the activation weight of expert $\mathcal{E}_{k}$ obtained from the activation scores $s^{(R)}$. The experts can be instantiated as a feed-forward block or a local attention operator like \cite{fedus2022switch}, ensuring efficient specialization across modalities.

\subsection{Director \textbf{\textit{coordinates shared feature structure}}}\label{subsec:director}
The director $\mathcal{D}$ is a coordinating module that aligns the output from the experts to a shared latent space, coordinating knowledge across modalities(Fig.\ref{fig:framework}-c) via two components:

\textbf{Group exponential moving average (GEMA)} The shared latent space in the director is derived through the compression of modality-specific representation learned by the experts via our GEMA algorithm. The director utilizes a block with the same structure as the expert. The top-$K$ activated experts contribute to the director update proportionally to their activation scores. Their normalized activation weights $\{\omega_k\}_{K}$ are accumulated and normalized to form expert contribution weights $\{\Omega\}_{R}$ within a batch $B$:
\begin{equation}
\begin{matrix}\Omega_{r}=\frac{\sum_{b=1}^{B}\omega_{b,r}}{\sum_{r'=1}^{R}\sum_{b=1}^{B}\omega_{b,r'}}. r=1,\dots,R.\end{matrix}
\end{equation}
The resulting $\Omega$ indicates the contribution of each expert in the batch, satisfying $\sum_{r=1}^{R}\Omega_{r}=1$. For parameters in the director $\eta$, a weighted average of them across all experts $\theta$ is computed by $\Omega$:
\begin{equation}
     \begin{matrix}\eta^{*}\leftarrow \mathcal{G}_{\text{EMA}}(\theta,\eta)=m\eta+(1-m)\sum_{r=1}^{R}\Omega_{r}\theta_r,\end{matrix}
\end{equation}
where $m$ is the momentum factor that gradually increases from 0.99 to 1.0 with a cosine rate in training \cite{chen2021empirical}. Our GEMA enables the director to gradually absorb and compress the modality-specific representations of the experts into a shared multi-modality latent space. By maintaining temporal consistency and smoothing cross-expert updates, GEMA stabilizes training and facilitates the emergence of shared feature structures across modalities within the DEX module.

\textbf{Alignment learning} The activated experts are optimized by the alignment loss $\mathcal{L}_{\text{co}}$ supervised by the director to coordinate their knowledge. The director produces the feature $f$ for $\hat{f}^{D}=\mathcal{D}(f)$, which indicates the shared knowledge in a shared space. The distance between the feature $\hat{f}^{E}$ from collaborative expert prediction and $\hat{f}^{D}$ from the director is measured and minimized to align the experts' representation to a shared space, enabling their coordination and discovery of the common knowledge. We utilize the cosine similarity as the loss $\mathcal{L}_{\text{co}}(\hat{f}^{E},\hat{f}^{D})=1-\frac{\hat{f}^{E}\hat{f}^{D}}{||\hat{f}^{E}||||\hat{f}^{D}||}$.

\begin{table*}
  \centering
  \caption{Comparison study shows our highest average performance across 26 downstream tasks covering 10 modalities. ``Avg.'' denotes the average score over all tasks. Colored cells mark the top-3 results within modalities, with \textcolor[HTML]{E9DAF3}{$\blacksquare$} highlighting the highest. The FMs marked with ``*'' are trained with paired text or labels, and the ``NI'' denotes the FMs pretrained on nature images.}\label{tab:comp}
  \resizebox{\linewidth}{!}
  {
  \begin{tabular}{l|c|ccc|cccc|cccc|ccccccccccccc}
    \textbf{Model}
    & \color{gray}-
    & MAE
    & DINOv2
    & DINOv3
    & PLIP*
    & CONCH*
    & USFM
    & RETFound
    & PanDerm
    & LVMMed
    & MedSAM*
    & MAE-MedVerse
    & \textbf{DEX}
    \\ 
    \textbf{Modality}
    & \color{gray}\textit{scratch}
    & NI
    & NI
    & NI
    & Path
    & Path
    & US
    & OCT
    & MM
    & MM
    & MM
    & MM
    & MM
    \\ 
    \textbf{Backbone}
    & \color{gray}ViT-B
    & ViT-B
    & ViT-B
    & ViT-B
    & ViT-B
    & ViT-B
    & ViT-B
    & ViT-L
    & ViT-B
    & ViT-B
    & ViT-B
    & ViT-B
    & ViT-B
    \\
    \textbf{Data size}
    & \color{gray}-
    & 1.28M
    & 120M
    & 1.7B
    & 208K
    & 1.17M
    & 2.19M
    & 736K
    & 2.15M
    & 1.3M
    & 1.57M
    & 4.97M
    & 4.97M
    \\
    \hline
    \textbf{CT} (2 tasks)
    & \color{gray}73.1
    & 80.8$_{+7.7}$
    & \cellcolor[HTML]{C4D2E7}83.4$_{+10.3}$
    & \cellcolor[HTML]{E9DAF3}\textbf{84.2$_{+11.1}$}
    & 79.2$_{+6.1}$
    & 82.5$_{+9.4}$
    & 77.4$_{+4.3}$
    & 77.9$_{+4.8}$
    & 83.1$_{+10.0}$
    & 81.7$_{+8.6}$
    & 78.2$_{+5.1}$
    & 80.4$_{+7.3}$
    & \cellcolor[HTML]{C4D2E7}83.9$_{+10.8}$
    \\
    \textbf{Endo} (2 tasks)
    & \color{gray}49.0
    & 62.2$_{+13.2}$
    & 60.8$_{+11.8}$
    & \cellcolor[HTML]{E9DAF3}\textbf{67.8$_{+18.8}$}
    & 55.9$_{+6.9}$
    & 62.2$_{+13.3}$
    & 59.5$_{+10.5}$
    & 59.5$_{+10.5}$
    & 61.9$_{+12.9}$
    & \cellcolor[HTML]{C4D2E7}64.5$_{+15.5}$
    & 64.2$_{+15.2}$
    & 59.7$_{+10.7}$
    & \cellcolor[HTML]{C4D2E7}65.0$_{+16.0}$
    \\
    \textbf{Fundus} (2 tasks)
    & \color{gray}63.2
    & \cellcolor[HTML]{C4D2E7}71.7$_{+8.5}$
    & 65.6$_{+2.4}$
    & 62.9$_{-0.3}$
    & 60.1$_{-3.1}$
    & 70.4$_{+7.2}$
    & 67.8$_{+4.6}$
    & 70.8$_{+7.6}$
    & 70.2$_{+7.0}$
    & 71.6$_{+8.4}$
    & \cellcolor[HTML]{C4D2E7}71.7$_{+8.5}$
    & \cellcolor[HTML]{C4D2E7}73.6$_{+10.4}$
    & \cellcolor[HTML]{E9DAF3}\textbf{75.7$_{+12.5}$}
    \\
    \textbf{MR} (2 tasks)
    & \color{gray}62.0
    & \cellcolor[HTML]{E9DAF3}\textbf{73.0$_{+11.0}$}
    & 72.6$_{+10.6}$
    & \cellcolor[HTML]{C4D2E7}72.8$_{+10.8}$
    & 70.9$_{+8.9}$
    & 70.1$_{+8.1}$
    & 67.7$_{+5.7}$
    & 69.8$_{+7.8}$
    & 72.1$_{+10.1}$
    & 68.6$_{+6.6}$
    & 65.3$_{+3.3}$
    & 65.1$_{+3.1}$
    & \cellcolor[HTML]{C4D2E7}72.9$_{+10.9}$
    \\
    \textbf{OCT} (2 tasks)
    & \color{gray}47.6
    & 76.4$_{+28.8}$
    & 80.9$_{+33.3}$
    & \cellcolor[HTML]{C4D2E7}85.0$_{+37.4}$
    & 68.6$_{+21.0}$
    & 80.2$_{+32.6}$
    & 73.5$_{+25.9}$
    & \cellcolor[HTML]{E9DAF3}\textbf{87.6$_{+40.0}$}
    & 81.0$_{+33.4}$
    & 80.9$_{+33.3}$
    & 67.8$_{+20.2}$
    & 79.7$_{+32.1}$
    & \cellcolor[HTML]{C4D2E7}87.5$_{+39.9}$
    \\
    \textbf{Path} (4 tasks)
    & \color{gray}47.1
    & 61.1$_{+14.0}$
    & 58.5$_{+11.4}$
    & 61.4$_{+14.3}$
    & 56.0$_{+8.9}$
    & \cellcolor[HTML]{C4D2E7}67.9$_{+20.8}$
    & 61.7$_{+14.6}$
    & 57.0$_{+9.9}$
    & 61.8$_{+14.7}$
    & \cellcolor[HTML]{C4D2E7}66.0$_{+18.9}$
    & 61.5$_{+14.4}$
    & 60.9$_{+13.8}$
    & \cellcolor[HTML]{E9DAF3}\textbf{69.9$_{+22.8}$}
    \\
    \textbf{X-ray} (6 tasks)
    & \color{gray}76.6
    & 84.0$_{+7.4}$
    & 84.1$_{+7.5}$
    & 84.7$_{+8.1}$
    & 83.1$_{+6.5}$
    & 84.4$_{+7.8}$
    & 83.3$_{+6.7}$
    & 83.6$_{+7.0}$
    & \cellcolor[HTML]{C4D2E7}84.9$_{+8.3}$
    & \cellcolor[HTML]{C4D2E7}84.9$_{+8.3}$
    & 83.8$_{+7.2}$
    & 84.1$_{+7.5}$
    & \cellcolor[HTML]{E9DAF3}\textbf{85.9$_{+9.3}$}
    \\
    \textbf{US} (3 tasks)
    & \color{gray}47.7
    & 79.7$_{+32.0}$
    & 76.5$_{+28.8}$
    & 80.9$_{+33.2}$
    & 56.1$_{+8.4}$
    & 77.6$_{+29.9}$
    & 70.1$_{+22.4}$
    & 75.0$_{+27.3}$
    & 77.2$_{+29.5}$
    & \cellcolor[HTML]{C4D2E7}80.1$_{+32.4}$
    & \cellcolor[HTML]{C4D2E7}81.1$_{+33.4}$
    & 77.2$_{+29.5}$
    & \cellcolor[HTML]{E9DAF3}\textbf{82.5$_{+34.8}$}
    \\
    \textbf{PET} (1 task)
    & \color{gray}7.6
    & 19.5$_{+11.9}$
    & 27.2$_{+19.6}$
    & 35.8$_{+28.2}$
    & 4.7$_{-2.9}$
    & 13.0$_{+5.4}$
    & 15.1$_{+7.5}$
    & \cellcolor[HTML]{C4D2E7}62.6$_{+55.0}$
    & \cellcolor[HTML]{E9DAF3}\textbf{65.0$_{+57.4}$}
    & 42.4$_{+34.8}$
    & \cellcolor[HTML]{C4D2E7}55.6$_{+48.0}$
    & 35.8$_{+28.2}$
    & 52.1$_{+44.5}$
    \\
    \textbf{Photo} (2 tasks)
    & \color{gray}72.8
    & 85.4$_{+12.6}$
    & 83.6$_{+10.8}$
    & \cellcolor[HTML]{E9DAF3}\textbf{88.1$_{+15.3}$}
    & 82.3$_{+9.5}$
    & \cellcolor[HTML]{C4D2E7}86.8$_{+14.0}$
    & 80.0$_{+7.2}$
    & 82.4$_{+9.6}$
    & 85.4$_{+12.6}$
    & 84.9$_{+12.1}$
    & 83.9$_{+11.1}$
    & 84.3$_{+11.5}$
    & \cellcolor[HTML]{C4D2E7}87.1$_{+14.3}$
    \\
    \hline
    \textbf{Avg.} (26 tasks)
    & \color{gray}59.0
    & 73.3$_{+14.3}$
    & 72.7$_{+13.7}$
    & 75.2$_{+16.2}$
    & 66.5$_{+7.5}$
    & 74.2$_{+15.2}$
    & 70.1$_{+11.1}$
    & 73.6$_{+14.6}$
    & \cellcolor[HTML]{C4D2E7}75.4$_{+16.4}$
    & \cellcolor[HTML]{C4D2E7}75.4$_{+16.4}$
    & 73.5$_{+14.5}$
    & 73.1$_{+14.1}$
    & \cellcolor[HTML]{E9DAF3}\textbf{78.4$_{+19.4}$}
  \end{tabular}
  }
\end{table*}

\subsection{DEX modular network architecture}\label{subsec:network}
The proposed DEX modular network hierarchically stacks DEX modules, enabling modules to progressively coordinate specialization–coordination dynamics along the representation evolution (Fig.\ref{fig:framework}-a). Following the Transformers \cite{dosovitskiy2020image,liu2021swin}, each input image is partitioned into patches with $p\times p$ size and linearly projected into patch embeddings. Fixed sine–cosine positional embeddings are added, forming the input tokens $f^{(N,C)}$, where $N$ is the token number and $C$ is the token dimension. These tokens are processed by our DEX module $\mathcal{M}$ to represent multi-modality knowledge. Stacking $L$ DEX modules constructs the full DEX modular network, enabling the self-organization across granularities. In practice, our DEX network inserts a self-attention layer before each DEX module to aggregate contextual information like the ViT \cite{dosovitskiy2020image}. 

Our DEX coordinates the modules within the network through three complementary losses. \textbf{1)} \textit{Alignment loss} $\mathcal{L}_{\text{co}}$ aligns the experts' representation into a shared latent space defined by the director, coordinating multi-modality knowledge for shared feature structures. \textbf{2)} \textit{Self-supervised learning loss} $\mathcal{L}_{\text{self}}$ (e.g., MoCo \cite{he2020momentum}, MAE \cite{he2022masked}, COVER \cite{He_2025_ICCV}) drives the network to learn generalizable representations. \textbf{3)} \textit{Balance loss} $\mathcal{L}_{\text{bal}}$ regulates the activation of experts to avoid their potential collapse into fixed modules following \cite{fedus2022switch}. The overall objective is
\begin{equation}
\begin{matrix}\mathcal{L}_{DEX}=\frac{\lambda_{\text{co}}}{L}\sum_{l=1}^{L}\alpha^{l}\mathcal{L}_{\text{co}}^{l}+\frac{\lambda_{\text{bal}}}{L}\sum_{l=1}^{L}\mathcal{L}_{\text{bal}}^{l}+\mathcal{L}_{\text{self}}\end{matrix},
\end{equation}
where $\lambda_{\text{bal}}$ and $\lambda_{\text{co}}$ denote the weights of the balance and alignment losses. We set $\lambda_{\text{bal}}=0.001$ to avoid over-regularization and $\lambda_{\text{align}}=0.1$ to prevent dominance over semantic learning. The layer-wise factor $\alpha^{l}=\frac{1}{L-(l-1)}$ progressively strengthens alignment in deeper layers, encouraging the network to form increasingly shared representations and gradually enlarge parameter subspaces overlap during evolution as described in Fig.~\ref{fig:modeling}.

\section{Experiment}
\subsection{Experiment protocol} This study has made sufficient experiments for a complete evaluation. Following is
 the overview protocol, and details are in \textit{Appendix}.
 
\textbf{1) MedVerse dataset for multi-modality MVFM pretraining} To facilitate large-scale multi-modality MVFM pretraining, we curate a new benchmark named \textbf{Medical Vision Universe (MedVerse)}. It comprises \textbf{4,973,080} high-quality 2D medical images collected from publicly available datasets and institutional repositories. For the data with 3D volumes, e.g., CT, we slice them into 2D images like \cite{mh2023lvm}. MedVerse spans \textbf{10} distinct imaging modalities, including X-ray, CT, OCT, PET, MR, Ultrasound (US), Fundus, Pathology (Path), Endoscopy (Endo), and Clinical photos. It covers a wide spectrum of anatomical regions and clinical scenarios, involving the chest, abdomen, head, and other frequently examined organs and tissues. Each image is carefully standardized in both format and resolution to ensure cross-modality consistency. By encompassing both radiological and non-radiological domains, MedVerse enables representation learning across diverse medical visual patterns and anatomical contexts, thereby providing an FM-level dataset foundation.

\textbf{2) Adaptation on 26 downstream tasks} The adaptation performance of pretrained MVFMs is evaluated on 26 downstream tasks that cover a wide spectrum of modalities, anatomical regions, and clinical objectives. These tasks span classification and segmentation, and involve 10 modalities consistent with our MedVerse dataset. They span over 19 organs or tissues, providing a comprehensive evaluation of model adaptability across diverse MV scenarios.

\textbf{3) Comparisons} This study benchmarks our DEX against 11 representative vision FMs from both general and medical domains, grouped into three categories. \textbf{a.} \textit{Natural Image FMs} (MAE \cite{he2022masked}, DINOv2 \cite{oquab2024dinov2}, DINOv3 \cite{simeoni2025dinov3}) that are pretrained on large natural image datasets, representing strong general-domain baselines for transfer evaluation. \textbf{b.} \textit{Uni-modality MVFMs} (PLIP \cite{huang2023visual}, CONCH \cite{lu2024visual}, USFM \cite{jiao2024usfm}, RETFound \cite{zhou2023foundation}) that are specialized for a specific imaging modality, serving as competitive domain-specific baselines. \textbf{c.} \textit{Multi-modality MVFMs} (PanDerm \cite{yan2025multimodal}, LVMMed \cite{mh2023lvm}, MedSAM \cite{ma2024segment}, MAE \cite{he2022masked} pretrained on our MedVerse) that are designed for multi-modality understanding, reflecting the state-of-the-art (SOTA) in multi-modality MVFMs. All models use the same backbone integration, training pipeline, and hyperparameter settings in downstream tasks for fair evaluation.

\textbf{4) Implementation and evaluation metrics} All experiments are implemented by PyTorch and optimized by AdamW with an initial learning rate of $10^{-4}$. The model is pretrained for 100 epochs on our MedVerse dataset with a 10-epoch warm-up phase and then fine-tuned on each downstream task. Pretraining is conducted on two NVIDIA H100 Tensor Core GPUs with a batch size of 224 per GPU, and an image size of $256\times 256$. We implement our DEX network with the same architecture as ViT-B \cite{dosovitskiy2020image} for broad technical compatibility and use the MLP layers as the experts and director. We utilize the MAE \cite{he2022masked} as the self-supervised learning loss $\mathcal{L}_{\text{self}}$. In downstream tasks, all segmentation tasks are implemented via the UNETR \cite{hatamizadeh2022unetr} and all classification tasks adopt a classification head appended to the backbone network. This study used the Dice similarity coefficient (DSC) for segmentation tasks and the F1 measure (F1) for classification tasks. We report the average performance (\%) of all tasks within each modality in Tab.\ref{tab:comp}, and detailed per-task results in a radar chart of Fig.\ref{fig:leida}.

\subsection{Comparison study}
Our DEX achieves the highest overall performance across 26 downstream tasks spanning 10 modalities (Tab.\ref{tab:comp}), demonstrating strong generalization across diverse MV domains with four key observations: \textbf{a.} Natural image FMs also enhance medical image performance, achieving over 10\% average improvement compared with the ``scratch''. Their large-scale pretraining on natural images provides universal visual priors that transfer effectively to medical domains. Notably, DINOv3, pretrained on 1.7 billion (B) images, ranks in the top-3 across five modalities, and even three of them are the highest. \textbf{b.} Uni-modality MVFMs exhibit strong gains in their source modalities and are able to cross-modality transfer. The CONCH and RETFound achieved top-3 performance on pathology and OCT (their source modalities), validating their domain-specific learning. They also generalize reasonably to related modalities (e.g., CONCH to Photo, RETFound to PET), due to shared low-level cues. \textbf{c.} Transfer limitations in uni-modality MVFMs stem from limited data diversity and source-modality quality. PLIP, trained on the smallest dataset (208K), shows the lowest Avg. (66.5\%). USFM performs poorly (70.1\%) due to the high noise and variability in US data, which restricts cross-modality generalization. \textbf{d.} Multi-modality MVFMs have superiority in the cross-modality generalization. These FMs take over the top-3 highest average scores, contributing 61.3\% of all top-3 scores across the modalities. Their exposure to diverse modalities fosters emergent cross-modal representations and strong transferability across medical domains. Notably, although MedSAM is trained with segmentation supervision, its performance is lower than LVMMed, despite comparable pretraining data, suggesting that strong supervision may bias the model toward region features, reducing its general-purpose ability.

\begin{figure}
  \centering
  \includegraphics[width=\linewidth]{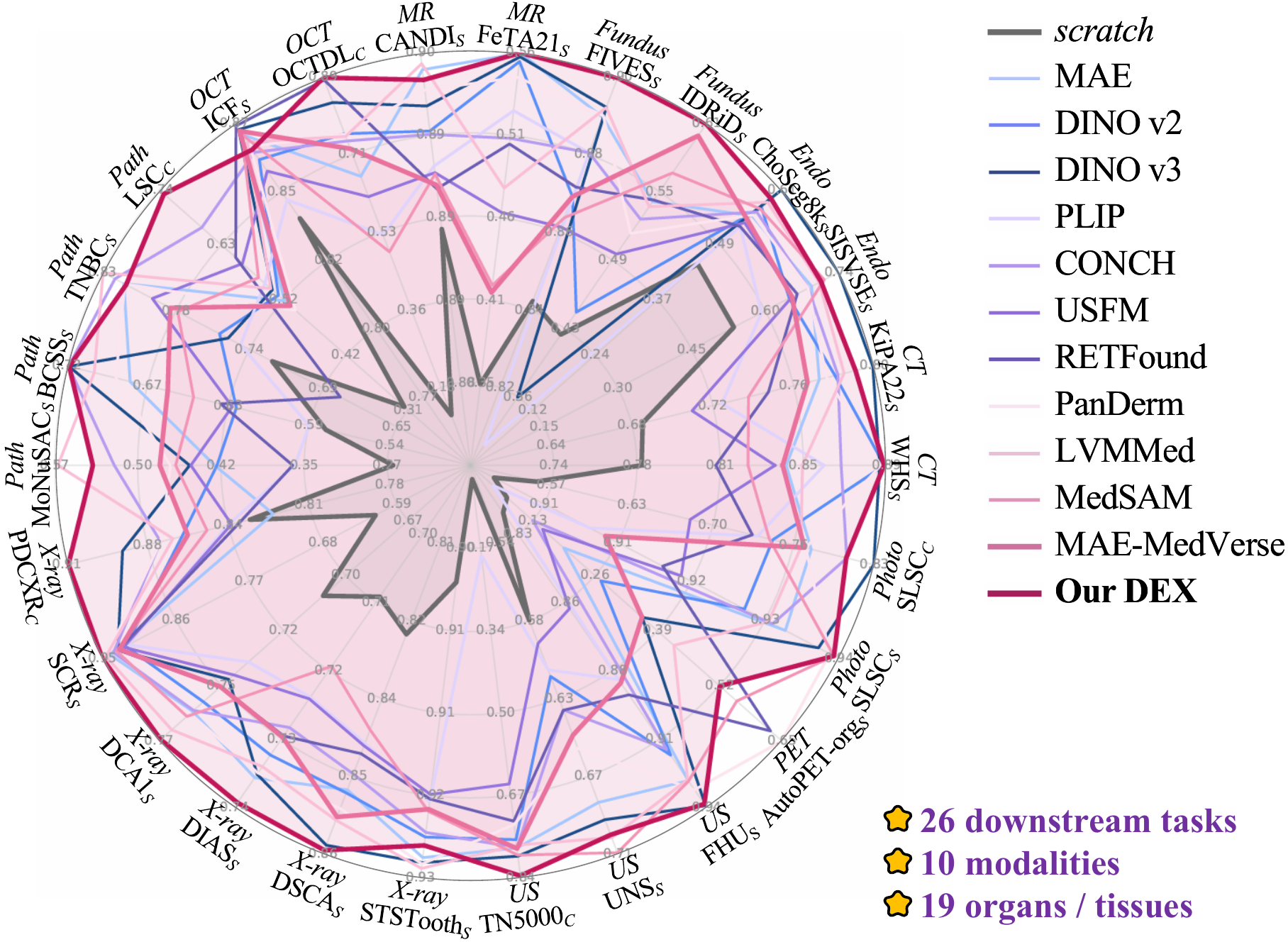}
  \caption{Our superiority in 26 downstream tasks across 10 modalities with 19 organs/tissues underscores its strong generalization ability across heterogeneous MV modalities.}
  \label{fig:leida}
\end{figure}
Compared with other multi-modality MVFMs, our DEX achieves the best overall performance, with two key observations: \textbf{a.} Strong generalization across modalities. DEX attains the highest average score (78.4\%), ranking within the top-3 across 9/10 modalities and first in four. This superior generalization arises from the model's explicit modeling of specialization–coordination dynamics, which enables the emergence of modular representations, improving adaptability in downstream tasks. \textbf{b.} Balanced and robust performance across tasks. DEX achieves the best results on 15 tasks and ranks within the top-3 across 22/26 tasks (Fig.\ref{fig:leida}). Although not the top performer on every individual task, DEX encloses the largest overall area, indicating balanced generalization across scenarios.
\begin{figure}
  \centering
  \includegraphics[width=\linewidth]{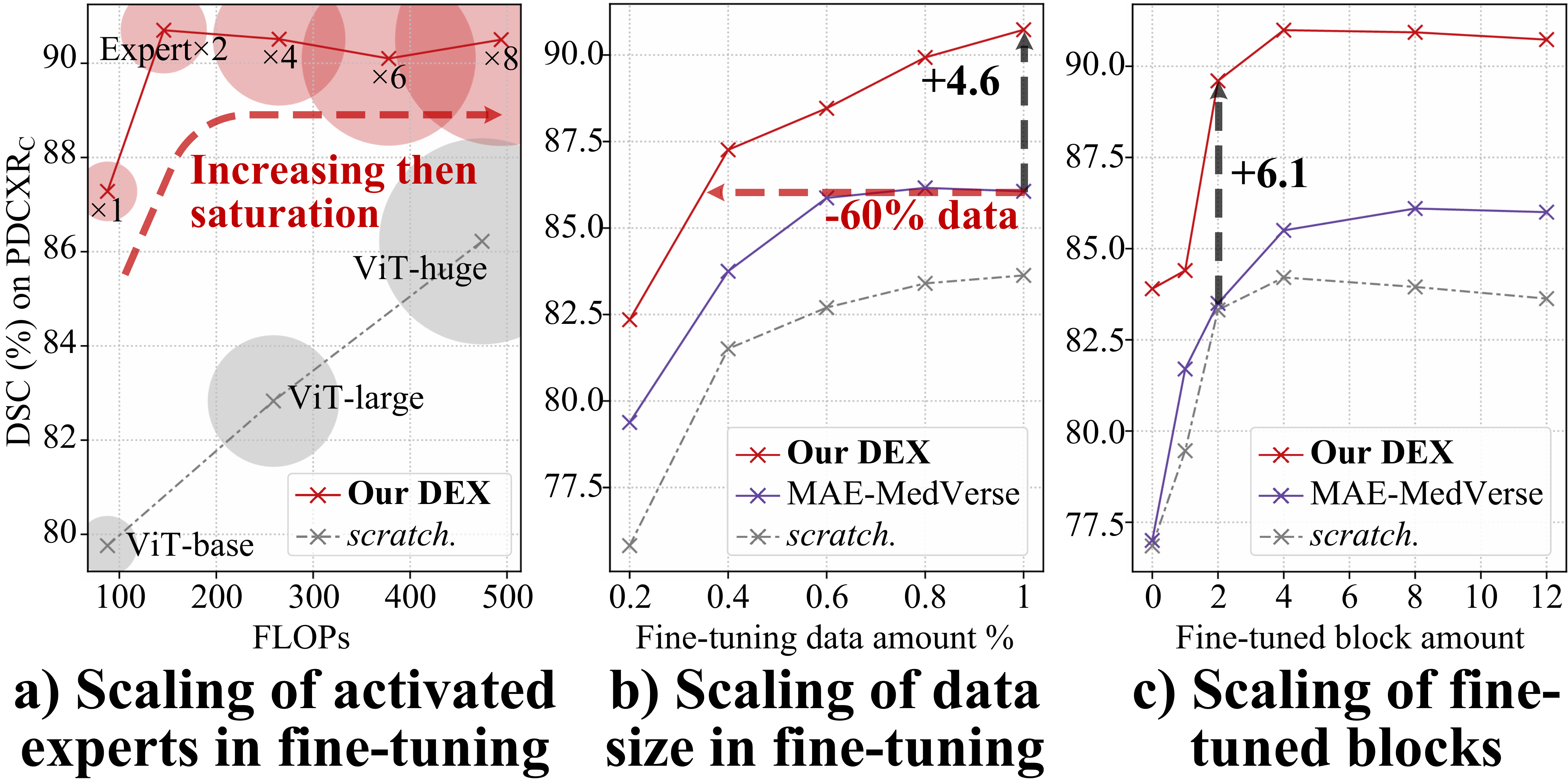}
  \caption{\textbf{Scaling analysis}. a) More activated experts in fine-tuning yield higher accuracy initially, followed by saturation. b) DEX achieves better data efficiency, outperforming scratch training with less fine-tuning data. c) DEX has more powerful performance with fewer tunable network blocks.}
  \label{fig:scale}
\end{figure}
\subsection{Ablation study and model analysis}
\textbf{1) Component ablation} The ablation study of our DEX on the PDCXR$_{C}$ (pneumonia classification (C) on X-ray) and TNBC$_{S}$ (nuclei segmentation (S) on pathology) shows the effectiveness of our experts and director in our model. We used ``scratch'' as the base and gradually added the MAE loss $\mathcal{L}_{\text{self}}$ on our MedVerse dataset, the experts for specialization, and the director for coordination. 
\begin{table}[H]
    \centering
    \resizebox{\linewidth}{!}
    {
    \begin{tabular}{cccccc}
        Modality
        & Task
        & \textit{scratch}
        & $\mathcal{L}_{\text{self}}$ (MedVerse)
        & + Experts
        & + Director
        \\
        \hline
        X-ray
        & PDCXR$_{C}$
        & 83.6
        & 86.1$_{+2.5}$
        & 88.4$_{+4.8}$
        & 90.7$_{+7.1}$
        \\
        Path
        & TNBC$_{S}$
        & 72.2
        & 78.4$_{+6.2}$
        & 78.5$_{+6.3}$
        & 81.1$_{+7.9}$
    \end{tabular}
    }
    \label{tab:placeholder}
\end{table}
\noindent Without any pretraining (``scratch''), the ViT-base backbone has 83.6\% F1 score and 72.2\% DSC on PDCXR$_{C}$ and TNBC$_{S}$. When incorporating pretraining on our MedVerse dataset, the performance improved by 2.5\% and 6.2\%, demonstrating that the rich modality knowledge contained in MedVerse effectively enhances downstream task performance. When adding our experts for specialization, it made the network a variant of a mixture of expert model \cite{fedus2022switch}, bringing 2.3\% improvement on PDCXR$_{C}$ owing to the specialized knowledge for X-ray images. However, it brings little gain on TNBC$_{S}$ since the severe domain gap prevents the activation from learning meaningful experts, causing the experts' specialization to collapse without proper coordination. When adding our Director, they serve as global coordinators that align experts' representations, effectively mitigating fragmented knowledge and substantially improving TNBC$_{S}$ and PDCXR$_{C}$ performance with 2.6\% and 2.3\%.
\begin{figure}
  \centering
  \includegraphics[width=\linewidth]{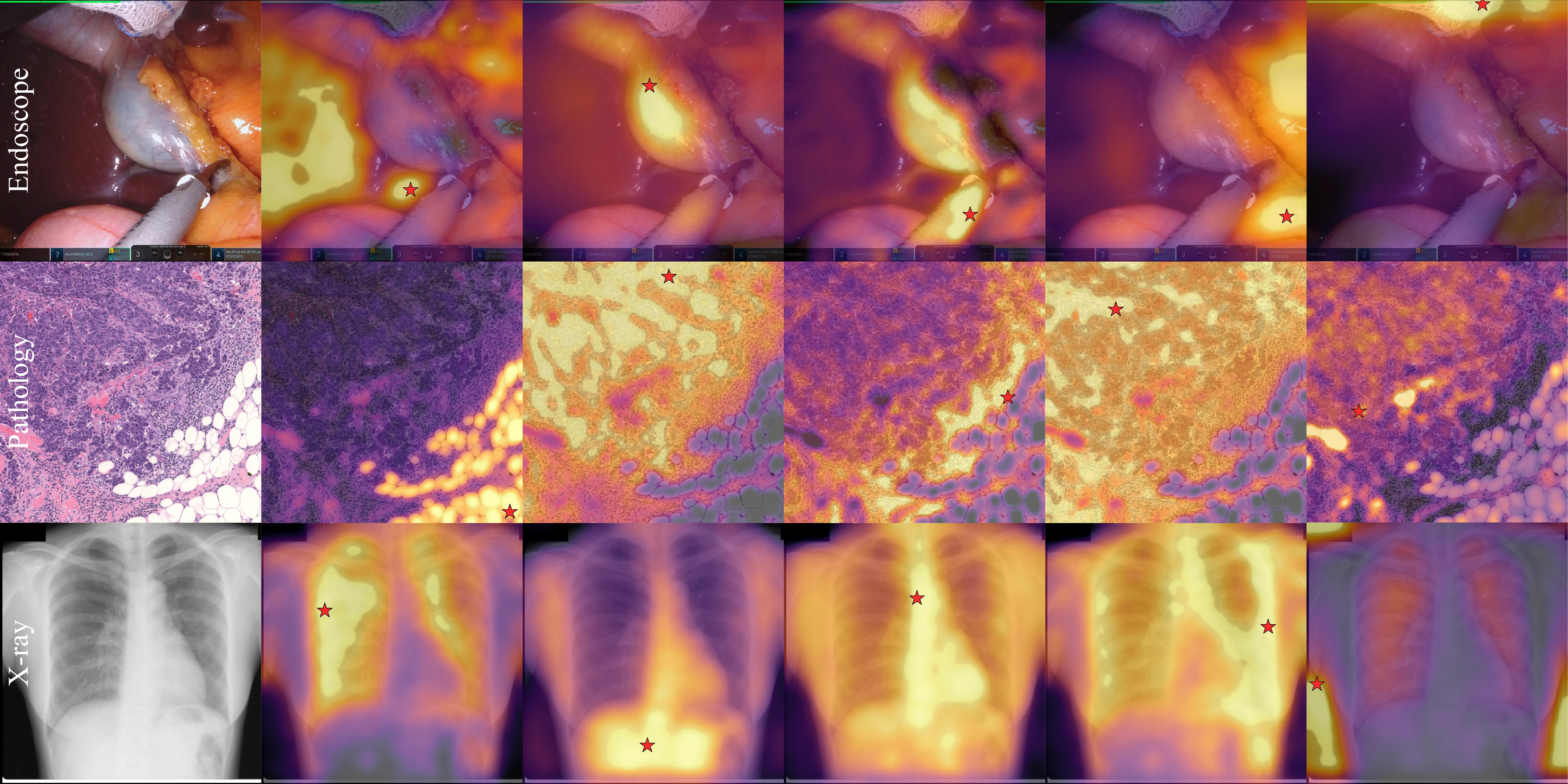}
  \caption{Visualization of attention maps from DEX. Each row denotes a modality, and each column shows the attention of a query patch ({\color{red}$\bigstar$}). DEX focuses on semantically meaningful regions, showing strong semantic decoupling ability.}
  \label{fig:heatmap}
\end{figure}

\textbf{2) Scaling analysis} This experiment analyzed three factors in the downstream fine-tuning (Fig.\ref{fig:scale}). \textbf{a)} \textit{Activated experts} ($K$) during fine-tuning: Increasing the number of activated experts leads to performance gains initially for DEX, which gradually saturates as $K$ grows. This trend confirms that in specialized tasks, although module integration is beneficial, larger integration does not bring equivalent gains due to the limitations in data amount and the singularity of the task. Although computational costs increase accordingly, ViT (our basic architecture) fails to achieve comparable performance even when scaled to a similar FLOPs. \textbf{b)} \textit{Fine-tuning data amount}: Our DEX exhibits strong data efficiency, achieving comparable performance to MAE-MedVerse using only 40\% of the fine-tuning data. Moreover, it surpasses both baselines by 4.6\% DSC with the full data, demonstrating superior generalization and transferability under less supervision. \textbf{c)} \textit{Fine-tuned blocks}: Our DEX network achieves a 6.1\% DSC improvement by fine-tuning only a small subset of DEX blocks, rapidly reaching its optimal performance. This indicates that DEX effectively leverages pretrained representations through localized adaptation, achieving stable optimization and efficient feature transfer compared to MAE-MedVerse.

\textbf{3) Attention analysis} We visualize the attention maps of the last-layer features in our DEX network to demonstrate its strong representational capability. Specifically, we select the feature of a patch and compute its attention map with all other features. To enhance visual clarity, we apply Gaussian smoothing for friendly visualization (Fig.\ref{fig:heatmap}). Our DEX exhibits clear and semantically consistent attention patterns across different visual modalities. In surgical scenes (top row), the attention maps accurately focus on key anatomical structures and surgical tools, indicating that the model effectively captures task-relevant spatial cues. In pathology images (middle row), DEX highlights discriminative tissue regions such as glandular boundaries and cellular clusters, suggesting that it learns fine-grained local representations critical for diagnosis. For X-ray images (bottom row), the attention is well aligned with major thoracic regions, such as the lungs and mediastinum, reflecting its capability to capture high-level semantic context. These visualizations demonstrate that DEX can attend to the semantically coherent regions, explaining the strong generalization performance observed in quantitative evaluations.

\textbf{4) Pattern layout of features} Figure \ref{fig:layout} visualizes the pattern layout learned within each modality by our DEX. Each block displays representative patches whose features are closely clustered in the learned latent space following \cite{zhouimage}. DEX is capable of discovering semantically consistent regions within modality: in endoscopic images, it groups patches depicting surgical instruments and fatty tissues; in pathology images, patches corresponding to lumen and stroma exhibit well-structured clustering patterns; and in CT and X-ray images, the model identifies coherent areas such as the lung, bone, and spine. These results indicate that DEX possesses strong intra-modal discriminability, that is able to organize local patterns into meaningful semantic layouts, thereby revealing its ability to capture semantic variations within each modality.
\begin{figure}
  \centering
  \includegraphics[width=\linewidth]{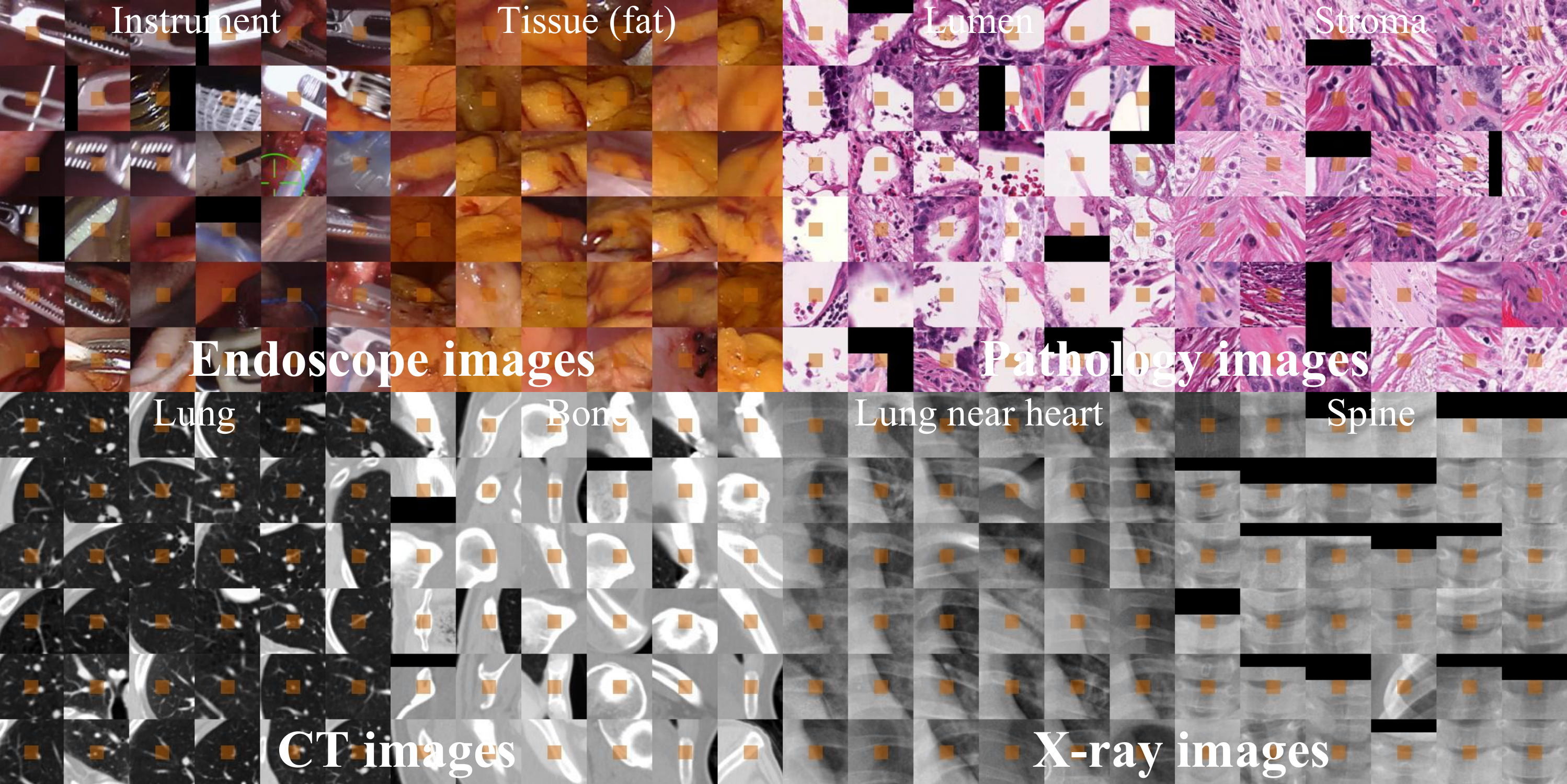}
  \caption{Visualization of intra-modality pattern layouts. Each block shows the patches that share similar features within modality. DEX clusters semantically related regions within each modality, demonstrating strong discriminability.}
  \label{fig:layout}
\end{figure}

\begin{figure}
  \centering
  \includegraphics[width=\linewidth]{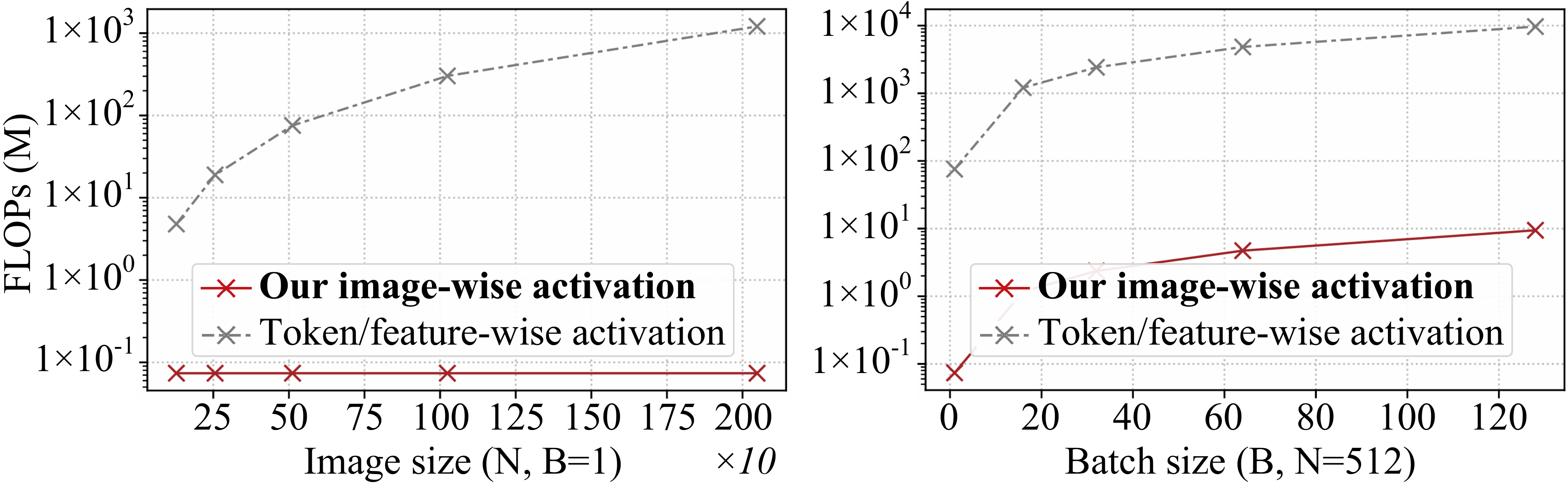}
  \caption{Computation cost of our image-wise activation. Across increasing image resolutions (left) and batch sizes (right), our image-wise activation incurs only a minimal, nearly constant overhead. In contrast, token/feature-wise activation grows rapidly due to its per-token dependency.}
  \label{fig:computation}
\end{figure}
\textbf{5) Computation Analysis of Image-wise \textit{v.s.} Token-wise Activation}
We analyzed the additional computation cost compared with ViT-base. We calculated the FLOPs of ViT-base $F_{\text{ViT}}$, our DEX based on our image-wise activation $F_{\text{DEX-I}}$, and our DEX based on token/feature-wise activation $F_{\text{DEX-T}}$. Then, we calculated their subtraction, i.e., $F_{\text{Image}}=F_{\text{DEX-I}}-F_{\text{ViT}}$, and $F_{\text{Token}}=F_{\text{DEX-T}}-F_{\text{ViT}}$. Our design exhibits two clear patterns (Fig.~\ref{fig:computation}): \textbf{a.} When increasing the input image resolution (left), token/feature-wise activation leads to a rapidly growing FLOPs overhead, increasing by several orders of magnitude as the number of tokens expands. In contrast, our image-wise activation introduces only a marginal and nearly constant cost, remaining close to the baseline ViT across all tested resolutions. \textbf{b.} When scaling the batch size (right), the overhead of token-wise activation again grows sharply, reflecting its per-token computation dependency. Our image-wise activation, however, scales gently and maintains a low computational footprint even under large-batch conditions. Together, these results demonstrate that the image-wise activation used in DEX is highly efficient and substantially more scalable than token-wise alternatives, making it suitable for high-resolution and large-batch medical imaging scenarios.
\begin{figure}
    \centering
    \includegraphics[width=\linewidth]{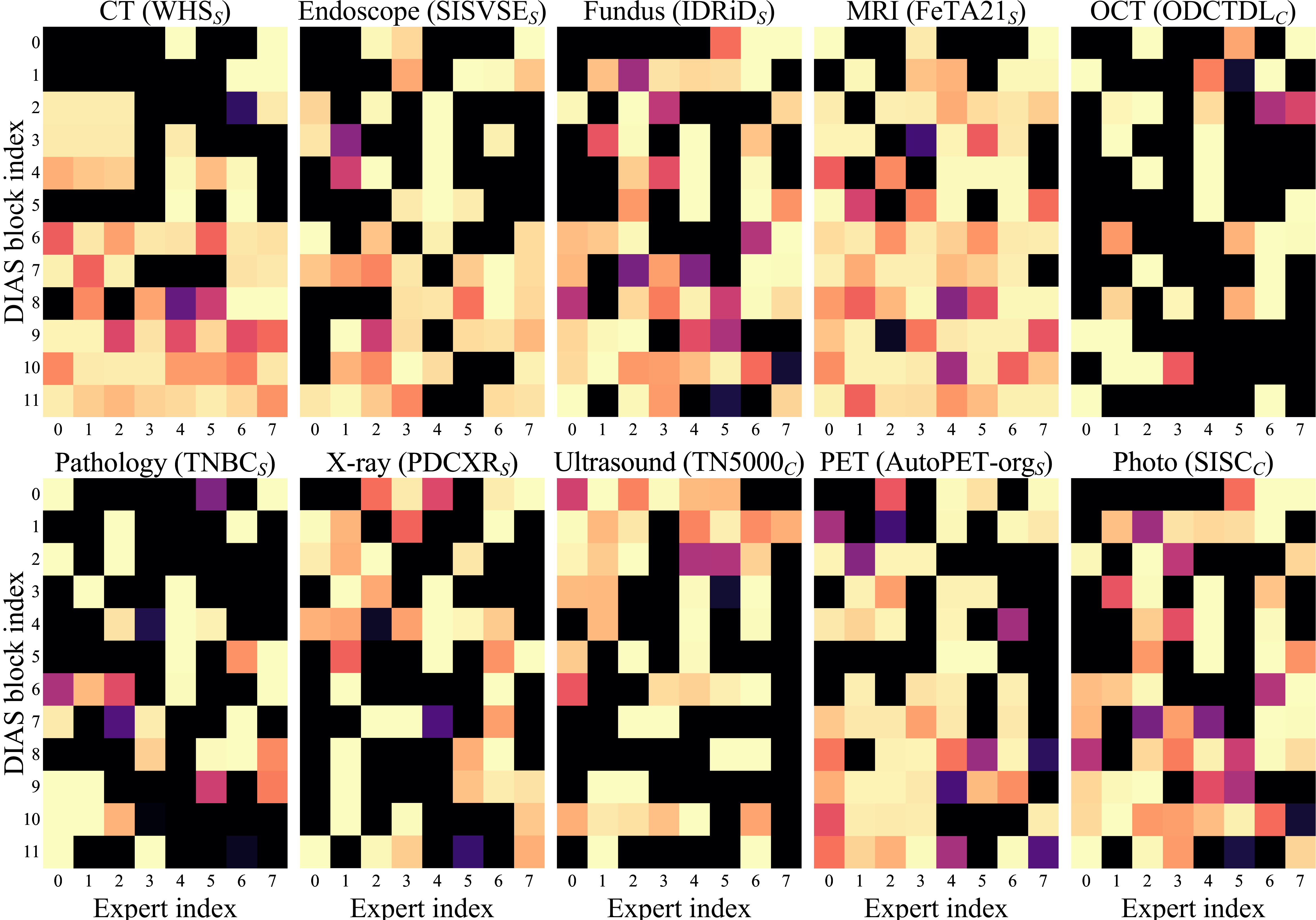}
    \caption{The visualization of the experts' activation frequency distribution in different modalities. }
    \label{fig:activation}
\end{figure}

\textbf{6) Experts Activation Analysis}
The activation frequency distribution of the experts across 10 modalities (Fig.\ref{fig:activation}) reveals several important patterns regarding how DEX dynamically allocates representation capacity: \textbf{a.} Modalities with strong structural regularities, e.g., CT, MRI, OCT, show concentrated activations on a small subset of experts, suggesting that DEX learns stable and reusable visual routines for these domains. \textbf{b.} Modalities with large appearance variability, e.g., Photo, Endoscope, trigger a more diverse set of experts, implying that DEX leverages a broader pool of experts to capture their heterogeneous texture and color distributions. \textbf{c.} Certain experts are consistently activated across multiple modalities, particularly those sharing low-level signals such as edge, contrast, or layered tissue patterns, e.g., Ultrasound and X-ray, OCT and Fundus. This indicates the emergence of modality-agnostic primitives that contribute to DEX's strong cross-domain generalization. \textbf{d.} Modalities with inherently noisy or complex acquisition characteristics, e.g., Ultrasound, PET, present more diffuse and less concentrated activation patterns, suggesting that DEX distributes representation load to mitigate noise sensitivity. Collectively, these findings demonstrate that DEX not only learns modality-specific experts but also discovers shared cross-modal experts, enabling flexible transfer across diverse medical imaging.

\section{Conclusion}
In this paper, we have proposed the DEX modular network for the Non-IID problem in multi-modality MVFMs. By reformulating FMs as modular networks, DEX achieves adaptive specialization and coordination across heterogeneous modalities, effectively mitigating representation collapse and enabling the emergence of modular representations. While its unique coordination properties exhibit strong performance across diverse medical modalities, an important future work is to extend the learning of emergent modular representations to high-dimensional, temporal, and cross-domain medical data for universal representation learning and involve more downstream evaluations. We believe that our DEX will promote the development of scalable, modality-unified MVFMs and inspire further research toward general-purpose intelligence.

\section*{Acknowledgment}
This work was supported in part by the National Institutes of Health (NIH) under Grants R01HL173186,  R01HL177813, and by the National Science Foundation (NSF) under Grant No. 2306545. This work used Jetstream2 through allocation CIS250356 from the Advanced Cyberinfrastructure Coordination Ecosystem: Services \& Support (ACCESS) program.

\bibliographystyle{ACM-Reference-Format}
\bibliography{sample-base}
\appendix
\clearpage
\section{Theoretical Foundation}
\subsection{Theoretical Modeling of the Non-IID Problem}
Multi-modality MVFM pretraining violates the IID assumption \cite{tishby2000information,cao2022beyond} due to heterogeneous imaging physics, protocols, and contexts \cite{yang2024vitkd,zong2024self,zhu2021federated}. This section provides a formal characterization of the Non-IID setting, its effect on self-supervised learning, and the resulting modality-dominant representation collapse.

\subsubsection{Preliminary 1: Multi-Modality Distributions}
Let $\mathcal{D}=\{1,\dots,D\}$ denote the set of imaging modalities, where each modality $d\in\mathcal{D}$ induces a distribution $P_d$ over input space $\mathcal{X}$. Uni-modality MVFM pretraining assumes IID samples $x_i\sim P$ drawn from a \emph{single} distribution $P$. Multi-modality pretraining samples from the mixture
\begin{equation}
    \begin{matrix}P_{\mathrm{mix}}=\sum_{d=1}^D \rho_d P_d, \rho_{d}\geq0, \sum_{d}\rho_{d}=1,\end{matrix}
\end{equation}
where $\rho_{d}$ is the sampling proportion of modality $d$. Each sample is effectively drawn as
\begin{equation}
d_i \sim \rho,\; x_i \sim P_{d_i},
\end{equation}
which violates the standard IID assumption that all $x_i\sim P$ for a single underlying $P$. This induces distribution shift across modalities, resulting in the \textit{Non-IID setting} that is characterized by pronounced inter-modality discrepancies and disjoint modality clusters in feature space \cite{ben2010theory}.

\begin{definition}[Non-IID Multi-Modality Setting]
A dataset is \emph{Non-IID} if its samples are drawn from a mixture of distributions $\{P_d\}_{d=1}^D$ satisfying $\exists d\neq d',\;P_d\neq P_{d'}$.
\end{definition}

\subsubsection{Preliminary 2: Self-supervised Learning in Non-IID setting}
Let $f_\theta:\mathcal{X}\rightarrow\mathbb{R}^d$ be a feature extractor, and let $\mathcal{T}$ denote the self-supervised pretext transformations. A self-supervised objective is formulated as
\begin{equation}
    \begin{matrix}\mathcal{L}_{\mathrm{self}}=\mathbb{E}_{x\sim P_{\mathrm{mix}}}[\ell(f_\theta(x),\,\mathcal{T}(x))] \\
    =\sum_{d=1}^D \rho_d \mathbb{E}_{x\sim P_d}[\ell(f_\theta(x), \mathcal{T}(x))]\end{matrix},
    \label{eq:ssl_mixture}
\end{equation}
where the loss $\ell$ induces a pretext task in the self-supervised learning (SSL). Standard SSL assumes IID sampling from a single $P$ \cite{tishby2000information}, which is violated under Non-IID sampling. 

\begin{definition}[Feature Decomposition]
\label{def:sem_mod}
We decompose the learned representation $f_{\theta}(x)$ as an arbitrary layer as a non-linear combination of a modality-invariant semantic component and a modality-specific component:
\begin{equation}
    f_\theta(x)= g_\theta^{\mathrm{sem}}(x)\oplus g_\theta^{\mathrm{mod}}(x,d),
\end{equation}
where $g_\theta^{\mathrm{sem}}(x)$ captures modality-invariant semantics, and $g_\theta^{\mathrm{mod}}(x,d)$ captures modality-specific biases, and $\oplus$ denotes integration.
\end{definition}

\subsubsection{Representation Collapse as Degenerate Optimum}
\begin{theorem}[Semantic Attenuation and Modality Collapse under Non-IID SSL]\label{thm:collapse}
Let $\triangle_{\mathrm{mod}}(\mathcal{F})$ denote the intrinsic distribution discrepancy across modalities in the hypothesis class $\mathcal{F}$. Without explicit control, i.e., optimizing $\mathcal{L}_{\mathrm{self}}$ only, as the modality discrepancy $\triangle_{\mathrm{mod}}(\mathcal{F})$ increases, the optimizer tends to prioritize modality-specific features. The relative magnitudes of the decomposed features will be:
\begin{equation}
    \frac{\|g^{\mathrm{mod}}_{\theta^*}\|}{\|g^{\mathrm{sem}}_{\theta^*}\|}\rightarrow \mathcal{O}(\triangle_{\mathrm{mod}}(\mathcal{F}))
\end{equation}
This semantic attenuation means that the network over-relies on $\mathcal{T}(x)$'s modality-dependent statistics for quick loss minimization, leading to representation collapse where effective generalization across modalities is lost. \end{theorem}

\begin{proof}[Proof Sketch]
In the Non-IID setting, minimizing the mixed objective $\mathcal{L}_{\mathrm{self}}$ (Equ.\eqref{eq:ssl_mixture}) is easier by finding shortcuts that reduce local modality risks $\mathbb{E}_{x\sim P_{d}}[\dots]$ independently. The modality-specific component $g^{\mathrm{mod}}_{\theta}$ offers a high-variance, low-effort path to minimize local $\mathcal{L}_{\mathrm{self}}$ terms. In contrast, the modality-invariant component $g^{\mathrm{sem}}_{\theta}$ must satisfy a consistent objective across all highly heterogeneous $P_{d}$, requiring significantly more optimization effort. This difference in gradient efficiency causes $g^{\mathrm{mod}}_{\theta}$ to be amplified and $g^{\mathrm{sem}}_{\theta}$ to be suppressed, leading to the collapsed representation dominated by modality biases.
\end{proof}

\subsection{Modular Networks for Non-IID Problem}\label{sec:iar_noniid}
Our Director-Experts (DEX) modular network, is designed to explicitly regulate the feature decomposition $f_{\theta}(x)=g^{\mathrm{sem}}_{\theta}(x)\oplus g_\theta^{\mathrm{mod}}(x,d)$ during optimization, thus counteracting semantic attenuation.

\begin{definition}[Director-Experts (DEX)]\label{def:iar}
Our DEX models a network $F$ as a hierarchical modular network $F=M_L \circ M_{L-1} \circ \dots \circ M_1$, where for an input $f_{l-1}$, each module $M_l$ ($l=1,\dots,L$) performs:
\begin{enumerate}
    \item \emph{\textbf{Specialization}} $\Phi^{l}_{\mathrm{spec}}$: Multiple expert modules dynamically partition the parameter space to learn modality-specific features $\hat{f}_{l}$.
    \item  \emph{\textbf{Coordination}} $\Phi^{l}_{\mathrm{co}}$: A director module coordinates experts' knowledge by aligning their output $\hat{f}_{l}$ to a shared latent space $\mathcal{D}(f_{l-1})$
\end{enumerate}
Excluding the other auxiliary losses, DEX is optimized by a self-supervised learning loss and an alignment loss:
\begin{equation}
    \mathcal{L}_{\mathrm{DEX}}=\mathcal{L}_{\mathrm{self}}+\frac{\lambda_{\mathrm{co}}}{L}\sum^{L}_{l=1}\alpha^{l}\mathcal{L}^{l}_{\mathrm{co}},
\end{equation}
where $\lambda_{\mathrm{align}}$ is the alignment weight and $\alpha^l = \frac{1}{L-(l-1)}$ is the layer-wise factor.
\end{definition}

\subsubsection{The Regularization Effect of Alignment Loss \texorpdfstring{$\mathcal{L}_{\mathrm{align}}$}{}}
The director module $\mathcal{D}$ maintains a representation $\hat{f}^D$ that serves as an aggregation of multi-modality knowledge via the Group Exponential Moving Average (GEMA) $\eta^*(\theta)=\mathcal{G}_{\mathrm{EMA}}(\theta,\eta)$. $\hat{f}^D$ thus models the stable, shared (semantic) component $g_\theta^{\mathrm{sem}}(x)$ across modalities. The alignment loss, based on cosine similarity, actively forces the expert's output $\hat{f}^E$ towards $\hat{f}^D$:
\begin{equation}
    \mathcal{L}_{\mathrm{align}}(\hat{f}^E,\hat{f}^D)=1-\frac{\hat{f}^{E}\hat{f}^{D}}{\|\hat{f}^{E}\|\|\hat{f}^{D}\|}
\end{equation}

\begin{lemma}[DEX Imposes Upper Bounds on Modality Residuals] \label{lem:iar_reg} The alignment loss $\mathcal{L}_{\mathrm{co}}$ effectively penalizes the modality-specific residual $g^{\mathrm{mod}}_{\theta}$. Specifically, the expected norm of the modality residual at layer $l$, $\mathbb{E}[\|g^{\mathrm{mod}}_{l}\|]$, is bounded by:
\begin{equation}
    \mathbb{E}[\|g^{\mathrm{mod}}_{l}\|]\leq K_{l}=\mathcal{O}(\frac{L_{\mathrm{self}}}{\lambda_{\mathrm{co}}\alpha^{l} c_{l}}),
\end{equation}
where $L_{\mathrm{self}}$ is the Lipschitz constant of $\mathcal{L}_{\mathrm{self}}$, $c_{l}>0$ is a constant related to the alignment loss property.

The factor $\alpha^l$ ensures that deeper layers have a greater penalty on modality residuals, enforcing progressive contraction towards a common representation, consistent with representation evolution.
\end{lemma}

\subsubsection{Mitigation of Semantic Attenuation}
\begin{theorem}[DEX Mitigates Semantic Attenuation under Non-IID SSL] \label{thm:iar_mitigates}
By setting the $\mathcal{L}_{\mathrm{co}}$ weights $\lambda_{\mathrm{co}}\alpha^l$ sufficiently large in deeper layers, DEX effectively contracts the modality discrepancy, reducing the overall effective discrepancy $\Delta_{\mathrm{mod}}^{\mathrm{DEX}}(\mathcal{F})$ compared to the original $\Delta_{\mathrm{mod}}(\mathcal{F})$. As a result, any minimizer $f_\theta^\star$ of $\mathcal{L}_{\mathrm{DEX}}$ satisfies:
\begin{equation} 
\|g_{\theta^\star}^{\mathrm{sem}}\|\ge C_{\mathrm{sem}}\quad\text{and}\quad\|g_{\theta^\star}^{\mathrm{mod}}\|\le C_{\mathrm{mod}}
\end{equation}
where $C_{\mathrm{sem}}>0$ is a lower bound independent of $\Delta_{\mathrm{mod}}(\mathcal{F})$, and $C_{\mathrm{mod}}$ is an upper bound. IAR thus simultaneously prevents semantic attenuation by requiring the Director to maintain a strong shared latent space, and prevents modality-dominated collapse by bounding the residual component.
\end{theorem}

\begin{proof}[Proof Sketch]
The expert modules use the self-supervised learning loss $\mathcal{L}_{\mathrm{self}}$ to learn high-variance, modality-specific features ($g_{\theta}^{\mathrm{mod}}$) while the Director, updated via the GEMA, learns the low-variance, modality-invariant features ($g_{\theta}^{\mathrm{sem}}$). The $\mathcal{L}_{\mathrm{align}}$ loss then acts as an explicit $\ell_2$-norm-like penalty on the distance between $\hat{f}^A$ and $\hat{f}^D$, effectively regularizing the magnitude of the $g_{\theta}^{\mathrm{mod}}$ (Lemma \ref{lem:iar_reg}), thus balancing the objective and ensuring that the representation is dominated by semantics, not modality-specific biases.
\end{proof}

\section{Technical Details}
\subsection{Balance Loss \texorpdfstring{$\mathcal{L}_{\mathrm{bal}}$}{}}
The balance loss $\mathcal{L}_{\mathrm{bal}}$ used in our DEX training will regulate the activation, avoiding the model collapsing to a fixed few experts and making specialization fail. Following \cite{fedus2022switch}, it is formulated as:
\begin{equation}
    \begin{matrix}\mathcal{L}_{\mathrm{bal}}(s^{(R)})=R\sum_{r=1}^{R} 
    (\frac{1}{B}\sum_{b=1}^{B} s_{b,r})
    (\frac{1}{B} \sum_{b=1}^{B}I_{b,r})\end{matrix},
    \label{eq:balance_loss}
\end{equation}
where $B$ is the batch size, $s_{b,r}$ denotes the soft activation scores of expert $r$ for image $b$, and $I_{b,r}\in\{0,1\}$ is the hard assignment indicator that equals $1$ if expert $r$ is selected for image $(b)$, otherwise $0$. This loss aligns the expected activation probability with their empirical activation frequencies, promoting balanced utilization of the experts \cite{fedus2022switch}.

\subsection{Adaptive Frequency-Aware Noise in Experts Activation}
During training, the module that selects Top-$K$ experts based on similarity logits may become biased toward a small subset of experts, especially at early stages when the activation scores are not yet well-formed. To alleviate this effect, we incorporate an adaptive frequency-aware noise mechanism that encourages a more balanced activation of experts without modifying the architecture or introducing additional learnable parameters. It has two components:

\textbf{1)} \textit{Momentum-based activation statistics.} After Top-$K$ activation, each image is associated with $K$ activated experts and corresponding activation weights $\{\omega\}_{K}$. We compute the per-batch activation distribution $\hat{c}=$ by summing these weights over all images and normalizing:
\begin{equation}
    \hat{c}_{r}=\frac{1}{\sum_{r'}\sum_{b}^{B}\omega_{b,r}+\varepsilon}\sum_{b=1}^{B}\omega_{b,r}, \quad r=1,\dots,R
\end{equation}
where $R$ is the number of experts, and $B$ is the batch size. This yields an activation distribution $\hat{c}=$ describing the relative contribution of each expert in the current batch. We maintain a stable long-term estimate using an exponential moving average:
\begin{equation}
    c\leftarrow\mu c+(1-\mu)\hat{c},
\end{equation}
where $\mu$ is the momentum, we set 0.99 in training.

\textbf{2)} \textit{Frequency-aware noise injection.} Before performing Top-$K$ activation, the activation scores $\{s\}_{R}$ are perturbed by expert-dependent noise:
\begin{equation}
    s^{b,r}=\mathrm{softmax(}(f_{b}\times\pi_{b,r})+\sigma(\epsilon_{b,r}+\delta_{r})), \quad\epsilon_{b,r}\sim\mathcal{N}(0,1),
\end{equation}
where $\sigma$ is the noise standard deviation, $\pi^{C}_{b,r}$ is the a learnable activation matrix, and $\delta$ is the frequency calculated from the activation distribution, i.e., $\delta_{r}=1-\frac{c_{r}}{\sum_{r'}^{R}c_{r'}+\varepsilon}$. It adjusts the competitiveness of each expert based on its long-term usage. Frequently activated experts (large $c_{r}$) receive weaker perturbations, while underused experts obtain stronger positive shifts, making them more likely to participate in the activation.

\section{Experiment Details}
\subsection{Datasets}
\subsubsection{Pretraining dataset: Medical Vision Universe (MedVerse)}
\begin{figure*}[htb]
  \centering
  \includegraphics[width=0.7\linewidth]{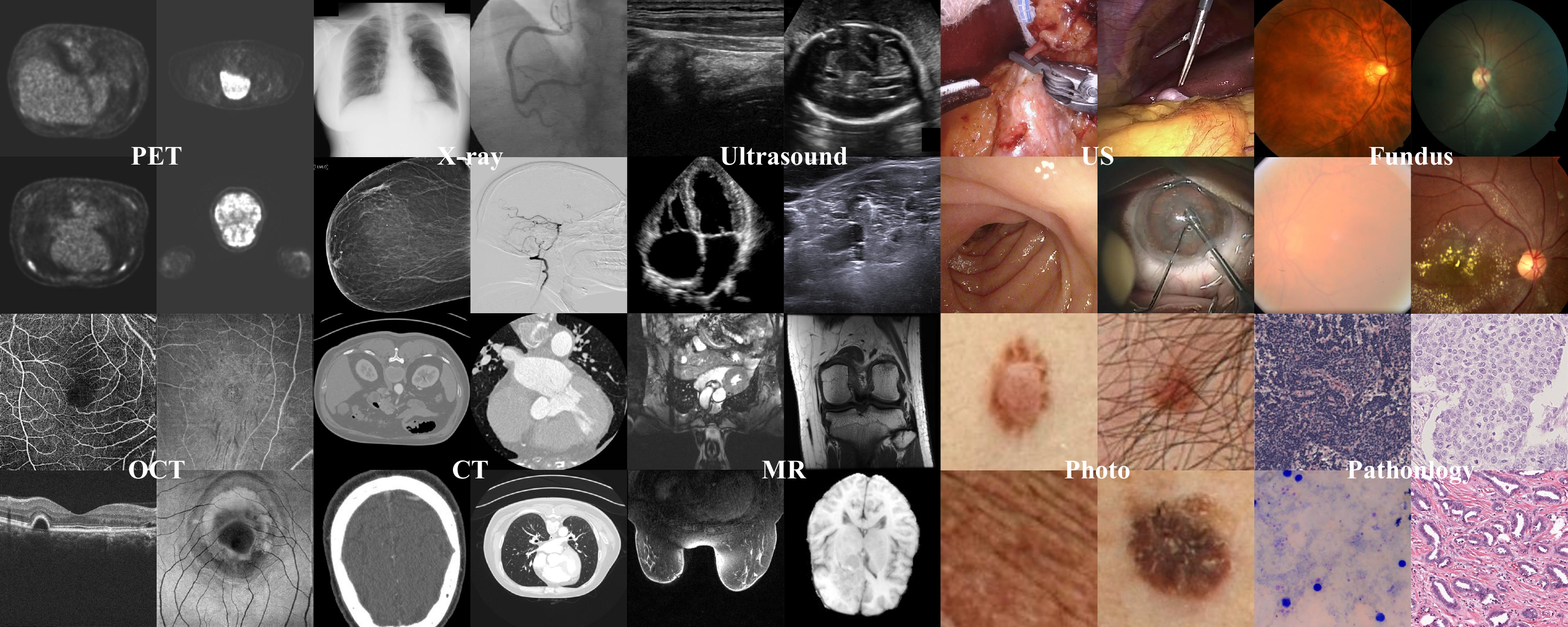}
  \caption{Examples in our MedVerse dataset. It has 10 modalities with diverse patterns and clinical targets, providing a large amount of medical image knowledge for FMs and enabling them to generalize to different tasks.}
  \label{fig:dataset}
\end{figure*}
Our curated MedVerse dataset is a large-scale multi-modality MV dataset collected from publicly available medical imaging datasets and aggregated to support training general-purpose MVFMs (Tab.\ref{tab:data_pre}). It contains more than 4.97 million medical images and slices across 10 modalities sourced from more than 40 well-established medical datasets, covering X-ray, CT, PET, MR, ultrasound, fundus, OCT, pathology, endoscopy, and clinical photography. It covers a broad spectrum of anatomical regions such as the chest, breast, bone, heart, brain, abdomen, eye, fetus, prostate, and skin. To improve the overall data integrity, we remove low-quality, corrupted, or abnormally small images, as well as highly repetitive or near-duplicate samples. This results in a cleaner, more diverse, and more representative corpus for pretraining.

\textbf{1) CT (821K):} The CT component of MedVerse comprises whole-body, cardiac, and brain CT datasets, including AutoPET-CT \cite{gatidis2024results}, ImageCAS \cite{zeng2023imagecas}, and TopBrain-CT \cite{yang2025benchmarking}, totaling more than 821K slices. To standardize voxel intensities across heterogeneous scanners and protocols, all CT volumes are intensity–normalized by clipping the Hounsfield Unit (HU) range to $[-200,400]$, which covers the majority of soft-tissue and organ structures relevant to downstream tasks. Following clipping, intensities are linearly rescaled to $[0,1]$ to ensure compatibility with other modalities in MedVerse. To further improve data quality, we discard slices with minimal anatomical content, defined as those with a foreground (non-air) pixel ratio below 0.05. This filtering step removes blank, near-empty, or scanner-padding slices commonly found at volume boundaries.

\textbf{2) PET (706K):} The PET portion of MedVerse from the AutoPET \cite{gatidis2024results} dataset comprising more than 706K slices. It includes both raw PET intensity volumes and standardized uptake value (SUV) volumes, enabling models to learn representations invariant to acquisition differences and thus improving cross-domain generalization. For raw PET data, we apply a three-stage normalization pipeline: (1) intensities are first clipped using the 1\% and 99\% percentile values to suppress extreme noise and scanner artifacts; (2) the clipped volumes are then standardized using z-score normalization; (3) the standardized values are finally mapped to $[0,1]$ via min–max normalization. For SUV volumes, we follow clinical conventions by clipping values to the range $[0,20]$, which captures typical physiological and pathological uptake levels, and subsequently linearly normalizing all values to $[0,1]$. To ensure anatomical relevance and remove non-informative slices from both modalities, we exclude slices in which the foreground ratio (non-air uptake) is below 0.01.

\textbf{3) MR (1.75M):} The MRI component of MedVerse consists of whole-body, cardiac, breast, abdominal, and brain MRI datasets, including Totalsegmenter-MR \cite{akinci2025totalsegmentator}, ACDC \cite{bernard2018deep}, MRNet \cite{bien2018deep}, Duke-Breast-Cancer-MRI \cite{saha2018machine}, Prostate-MRI-US-Biopsy \cite{fedorov2021nci,natarajan2020prostate}, TopBrain-MR \cite{yang2025benchmarking}, LLDMMR \cite{lou2025sdr}, and BraTS-GLI \cite{de20242024}, comprising more than 1.75M slices. To standardize voxel intensities across heterogeneous sources, all MRI volumes undergo intensity normalization by clipping each volume's intensity distribution using the 1\% and 99\% percentile values, which suppresses sequence-specific outliers and scanner-dependent artifacts. Following clipping, voxel intensities are linearly rescaled to $[0,1]$, facilitating consistent numerical ranges across different MRI contrasts and ensuring compatibility with other modalities in MedVerse. We discard slices with minimal anatomical content, defined as those with a foreground (non-air/tissue) pixel ratio below 0.05, thereby removing blank or non-informative slices typically appearing at volume boundaries.

\textbf{4) X-ray (468K):} The X-ray component of MedVerse consists of a diverse collection of large-scale public datasets, including ChestX-ray8 \cite{wang2017chestx}, CheXpert \cite{irvin2019chexpert}, Mammo-Bench \cite{bhole2025mammo}, MURA \cite{rajpurkar2017mura}, vzrad2 \cite{vzrad2_dataset}, ADSD \cite{danilov2021real}, ARCADE \cite{popov2024dataset}, and XACD \cite{ma2021self}, totaling more than 468K radiographs. These datasets span a wide range of anatomical regions, such as the chest, breast, bone, tooth, and coronary artery, and include both diagnostic and screening scenarios. All images are normalized by dividing pixel intensities by 255, ensuring consistent numerical ranges across sources.

\textbf{5) Ultrasound (68K):} The ultrasound component of MedVerse covers multiple anatomical regions, including cardiac, breast, thyroid, fetal, and peripheral nerve images, and aggregates both static images and cine sequences from public datasets such as EchoNet-Dynamic \cite{ouyang2019echonet,ouyang2020video} and related ultrasound benchmarks \cite{ThyroidUltrasoundCineclip,van2018automated,al2020dataset,ding2022mallesnet}, thus comprising more than 68K images. To make video-based data compatible with our 2D pretraining pipeline, we subsample ultrasound videos by storing one frame every 30 frames, which preserves the temporal diversity of cardiac motion while controlling redundancy and storage cost. All ultrasound images and sampled video frames are normalized by dividing pixel intensities by 255, resulting in values in the range $[0,1]$ and ensuring numerical consistency across datasets and acquisition devices. 

\textbf{6) Fundus (96K):} The fundus component of MedVerse includes a diverse set of ophthalmology datasets such as Refuge2 \cite{fang2022refuge2,orlando2020refuge}, MESSIDOR \cite{decenciere2014feedback}, APTOS2019 \cite{aptos2019}, and Eyepacs \cite{gulshan2016development}, comprising more than 96K images. These datasets provide high-resolution color fundus photographs captured under varying illumination conditions, camera systems, and clinical screening settings. To ensure consistent numerical ranges across heterogeneous sources, all fundus images are normalized by dividing RGB pixel intensities by 255, mapping them to the $[0,1]$ range while preserving color fidelity.

\textbf{7) OCT (110K):} The OCT component of MedVerse consists of retinal OCT datasets, including OCTA500 \cite{li2024octa}, ROSE \cite{ma2020rose}, and other public OCT benchmarks \cite{kermany2018identifying}, comprising more than 110K images. These datasets provide cross-sectional scans acquired under varying scanning protocols, resolutions, and device manufacturers. To standardize appearance across sources, all OCT images are normalized by dividing pixel intensities by 255, mapping values to the $[0,1]$ range while preserving relative contrast between retinal layers and lesions.

\textbf{8) Pathology (168K):} The pathology component of MedVerse includes cytology and histopathology datasets, covering the malaria cell datasets \cite{ljosa2012annotated,DVN/VEADSE_2023} and tissue microscopy dataset \cite{komura2023restaining}, thus comprising more than 168K images. These datasets contain high-resolution RGB microscopy images captured under diverse staining protocols, magnifications, and laboratory environments. To ensure consistent intensity scaling across heterogeneous microscopy sources, all pathology images are normalized by dividing pixel intensities by 255, producing standardized $[0,1]$ values while preserving fine-grained cellular contrast and chromatic variations.

\textbf{9) Endoscope (266K):} The endoscope component of MedVerse includes gastrointestinal, urological, and surgical endoscopy datasets, covering HyperKvasir \cite{borgli2020hyperkvasir}, Cataract-101 \cite{schoeffmann2018cataract}, PitVis2023 \cite{das2024pitvis}, Endoscapes-CVS \cite{murali2023endoscapes} and PSI-AVA \cite{valderrama2022towards}, thus comprising more than 266K images. It provides high-resolution RGB endoscopic frames acquired under diverse illumination settings, endoscope models, and clinical procedures. For video-based datasets, we subsample one frame every 30 frames to reduce temporal redundancy while retaining representative appearance variations across procedures. To ensure consistent intensity scaling across heterogeneous sources, all endoscopic images are normalized by dividing pixel intensities by 255, mapping them to the $[0,1]$ range while preserving fine-grained mucosal texture characteristics.

\textbf{10) Clinical photography (401K):} The clinical photography component of MedVerse primarily consists of skin-lesion datasets from the ISIC archive \cite{kurtansky2024slice} with over 401K images. It contains RGB clinical photographs captured under diverse lighting conditions, camera devices, and patient demographics, covering a wide range of benign and malignant skin lesions. To ensure consistent intensity scaling across heterogeneous sources, all clinical photographs are normalized by dividing pixel intensities by 255, mapping them to the $[0,1]$ range while preserving essential color cues for lesion morphology and pigmentation analysis. 
\begin{table*}[htb]
    \centering
    \caption{The composition of our Medical Vision Universe (MedVerse) dataset. It covers 10 medical image modalities, including more than 4M images, and is able to serve the pretraining of large-scale medical vision MMFMs.}
    \label{tab:data_pre}
  {
    \begin{tabular}{lccccc}
        \textbf{Dataset}
        & \textbf{Modality}
        & \textbf{Anatomy}
        & \textbf{Images (Slices)}
        & \textbf{Normalization}
        & \textbf{Link}
        \\
        \hline
        AutoPET-CT \cite{gatidis2024results}
        & CT
        & Whole body
        & 558,473
        & $\frac{\text{clip}(x,[-200,400])+200}{600}$
        & \href{https://autopet.grand-challenge.org/Dataset/}{$\bigstar$}
        \\
        ImageCAS \cite{zeng2023imagecas}
        & CT
        & Heart
        & 257,496
        & $\frac{\text{clip}(x,[-200,400])+200}{600}$
        & \href{https://www.kaggle.com/datasets/xiaoweixumedicalai/imagecas/}{$\bigstar$}
        \\
        Topbrain-CT \cite{yang2025benchmarking}
        & CT
        & Brain
        & 5,647
        & $\frac{\text{clip}(x,[-200,400])+200}{600}$
        & \href{https://zenodo.org/records/16623496}{$\bigstar$}
        \\
        AutoPET-PET \cite{gatidis2024results}
        & PET
        & Whole body
        & 354,826
        & $\frac{\frac{\text{clip}(x,[p_{1\%},p_{99\%}])-\mu}{\sigma}-x_{min}}{x_{max}-x_{min}}$
        & \href{https://autopet.grand-challenge.org/Dataset/}{$\bigstar$}
        \\
        AutoPET-SUV \cite{gatidis2024results}
        & PET (SUV)
        & Whole body
        & 351,900
        & $\frac{\text{clip}(x,[0,20])}{20}$
        & \href{https://autopet.grand-challenge.org/Dataset/}{$\bigstar$}
        \\
        Totalsegmenter-MR \cite{akinci2025totalsegmentator}
        & MR
        & Whole body
        & 16,059
        & $\frac{\text{clip}(x,[p_{1\%},p_{99\%}])-x_{min}}{x_{max}-x_{min}}$
        & \href{https://zenodo.org/records/14710732}{$\bigstar$}
        \\
        ACDC \cite{bernard2018deep}
        & MR
        & Heart
        & 1,902
        & $\frac{\text{clip}(x,[p_{1\%},p_{99\%}])-x_{min}}{x_{max}-x_{min}}$
        & \href{https://www.creatis.insa-lyon.fr/Challenge/acdc/}{$\bigstar$}
        \\
        MRNet \cite{bien2018deep}
        & MR
        & Knee
        & 116,624
        & $\frac{\text{clip}(x,[p_{1\%},p_{99\%}])-x_{min}}{x_{max}-x_{min}}$
        & \href{https://aimi.stanford.edu/datasets/mrnet-knee-mris}{$\bigstar$}
        \\
        Duke-Breast-Cancer-MRI \cite{saha2018machine}
        & MR
        & Breast
        & 515,747
        & $\frac{\text{clip}(x,[p_{1\%},p_{99\%}])-x_{min}}{x_{max}-x_{min}}$
        & \href{https://www.cancerimagingarchive.net/collection/duke-breast-cancer-mri/}{$\bigstar$}
        \\
        Prostate-MRI-US-Biopsy \cite{fedorov2021nci,natarajan2020prostate}
        & MR
        & Prostate
        & 97,180
        & $\frac{\text{clip}(x,[p_{1\%},p_{99\%}])-x_{min}}{x_{max}-x_{min}}$
        & \href{https://www.cancerimagingarchive.net/collection/prostate-mri-us-biopsy/}{$\bigstar$}
        \\
        TopBrain-MR \cite{yang2025benchmarking}
        & MR
        & Brain
        & 4,604
        & $\frac{\text{clip}(x,[p_{1\%},p_{99\%}])-x_{min}}{x_{max}-x_{min}}$
        & \href{https://zenodo.org/records/16623496}{$\bigstar$}
        \\
        BraTS-GLI \cite{de20242024}
        & MR
        & Brain
        & 802,341
        & $\frac{\text{clip}(x,[p_{1\%},p_{99\%}])-x_{min}}{x_{max}-x_{min}}$
        & \href{https://www.synapse.org/Synapse:syn53708249/wiki/627491}{$\bigstar$}
        \\
        LLDMMR \cite{lou2025sdr}
        & MR
        & Abdomen
        & 198,918
        & $\frac{\text{clip}(x,[p_{1\%},p_{99\%}])-x_{min}}{x_{max}-x_{min}}$
        & \href{https://github.com/LMMMEng/LLD-MMRI-Dataset}{$\bigstar$}
        \\
        ChestX-ray8 \cite{wang2017chestx}
        & X-ray
        & Chest
        & 112,120
        & x/255
        & \href{https://nihcc.app.box.com/v/ChestXray-NIHCC}{$\bigstar$}
        \\
        CheXpert \cite{irvin2019chexpert}
        & X-ray
        & Chest
        & 223,648
        & x/255
        & \href{https://stanfordmlgroup.github.io/competitions/chexpert/}{$\bigstar$}
        \\
        Mammo-Bench \cite{bhole2025mammo}
        & X-ray
        & Breast
        & 71,844
        & x/255
        & \href{https://india-data.org/dataset-details/c86fb00c-0fb8-4e0e-85a2-4d415f9c1ada}{$\bigstar$}
        \\
        MURA \cite{rajpurkar2017mura}
        & X-ray
        & Bone
        & 40,005
        & x/255
        & \href{https://stanfordmlgroup.github.io/competitions/mura/}{$\bigstar$}
        \\
        vzrad2 \cite{vzrad2_dataset}
        & X-ray
        & Tooth
        & 8,188
        & x/255
        & \href{https://www.kaggle.com/datasets/henriquerezermosqur/dental-x-ray-computacional-vision-segmentation/data}{$\bigstar$}
        \\
        ADSD \cite{danilov2021real}
        & X-ray
        & Coronary artery
        & 8,325
        & x/255
        & \href{https://data.mendeley.com/datasets/ydrm75xywg/2}{$\bigstar$}
        \\
        ARCADE \cite{popov2024dataset}
        & X-ray
        & Coronary artery
        & 3,000
        & x/255
        & \href{https://zenodo.org/records/10390295}{$\bigstar$}
        \\
        XACD \cite{ma2021self}
        & X-ray
        & Coronary artery
        & 1,621
        & x/255
        & \href{https://github.com/AISIGSJTU/SSVS}{$\bigstar$}
        \\
        Thyroid Ultrasound Cine-clip \cite{ThyroidUltrasoundCineclip}
        & Ultrasound
        & Thyroid
        & 1,737
        & x/255
        & \href{https://aimi.stanford.edu/datasets/thyroid-ultrasound-cine-clip}{$\bigstar$}
        \\
        EchoNet-Dynamic \cite{ouyang2019echonet,ouyang2020video}
        & Ultrasound
        & Heart
        & 63,867
        & x/255
        & \href{https://echonet.github.io/dynamic/}{$\bigstar$}
        \\
        HC18 \cite{van2018automated}
        & Ultrasound
        & Fetal head
        & 1,334
        & x/255
        & \href{https://hc18.grand-challenge.org/}{$\bigstar$}
        \\
        BUSI \cite{al2020dataset}
        & Ultrasound
        & Breast
        & 780
        & x/255
        & \href{https://scholar.cu.edu.eg/?q=afahmy/pages/dataset}{$\bigstar$}
        \\
        UBPD \cite{ding2022mallesnet}
        & Ultrasound
        & Brachial plexus
        & 955
        & x/255
        & \href{https://ubpd.worldwidetracing.com:9443/}{$\bigstar$}
        \\
        MESSIDOR \cite{decenciere2014feedback}
        & Fundus
        & Eye
        & 1,200
        & x/255
        & \href{https://www.adcis.net/en/third-party/messidor/}{$\bigstar$}
        \\
        Refuge2 \cite{fang2022refuge2,orlando2020refuge}
        & Fundus
        & Eye
        & 1,200
        & x/255
        & \href{https://www.kaggle.com/datasets/victorlemosml/refuge2}{$\bigstar$}
        \\
        APTOS2019 \cite{aptos2019}
        & Fundus
        & Eye
        & 5,590
        & x/255
        & \href{https://ahmadelsallab.github.io/APTOS/}{$\bigstar$}
        \\
        Eyepacs \cite{gulshan2016development}
        & Fundus
        & Eye
        & 88,702
        & x/255
        & \href{https://huggingface.co/datasets/bumbledeep/eyepacs}{$\bigstar$}
        \\
        ROSE \cite{ma2020rose}
        & OCT
        & Eye
        & 229
        & x/255
        & \href{https://imed.nimte.ac.cn/dataofrose.html}{$\bigstar$}
        \\
        OCTA500-OCTA \cite{li2024octa}
        & OCT
        & Eye
        & 600
        & x/255
        & \href{https://ieee-dataport.org/open-access/octa-500}{$\bigstar$}
        \\
        OCTA500-OCT \cite{li2024octa}
        & OCT
        & Eye
        & 600
        & x/255
        & \href{https://ieee-dataport.org/open-access/octa-500}{$\bigstar$}
        \\
        ZhangLabData-OCT \cite{kermany2018identifying}
        & OCT
        & Eye
        & 109,309
        & x/255
        & \href{https://data.mendeley.com/datasets/rscbjbr9sj/3}{$\bigstar$}
        \\
        BBBC041v1 \cite{ljosa2012annotated}
        & Pathology
        & Blood smears
        & 1,328
        & x/255
        & \href{https://bbbc.broadinstitute.org/BBBC041}{$\bigstar$}
        \\
        Lacuna Malaria Detection Challenge \cite{ljosa2012annotated}
        & Pathology
        & Blood smears
        & 3,925
        & x/255
        & \href{https://zindi.africa/competitions/lacuna-malaria-detection-challenge}{$\bigstar$}
        \\
        Lacuna Malaria Datasets \cite{DVN/VEADSE_2023}
        & Pathology
        & Blood smears
        & 5,059
        & x/255
        & \href{https://dataverse.harvard.edu/dataset.xhtml?persistentId=doi:10.7910/DVN/VEADSE}{$\bigstar$}
        \\
        SegPath \cite{komura2023restaining}
        & Pathology
        & 8 kinds of cells
        & 158,687
        & x/255
        & \href{https://dakomura.github.io/SegPath/}{$\bigstar$}
        \\
        HyperKvasir \cite{borgli2020hyperkvasir}
        & Endoscope
        & Stomach, intestines
        & 110,079
        & x/255
        & \href{https://datasets.simula.no/hyper-kvasir/}{$\bigstar$}
        \\ 
        Cataract-101 \cite{schoeffmann2018cataract}
        & Endoscope
        & Eye
        & 42,157
        & x/255
        & \href{https://ftp.itec.aau.at/datasets/ovid/cat-101/}{$\bigstar$}
        \\
        PitVis2023 \cite{das2024pitvis}
        & Endoscope
        & Pituitary
        & 96,272
        & x/255
        & \href{https://rdr.ucl.ac.uk/articles/dataset/PitVis_Challenge_Endoscopic_Pituitary_Surgery_videos/26531686}{$\bigstar$}
        \\
        Endoscapes-CVS \cite{murali2023endoscapes}
        & Endoscope
        & Gall bladder
        & 55,783
        & x/255
        & \href{https://github.com/CAMMA-public/Endoscapes}{$\bigstar$}
        \\
        PSI-AVA \cite{valderrama2022towards}
        & Endoscope
        & Prostate
        & 72,160
        & x/255
        & \href{https://drive.google.com/drive/folders/1dSVOuSbfudRgod8cLlONm2VOv35HEIBj}{$\bigstar$}
        \\
        ISIC challenge 2024 training \cite{kurtansky2024slice}
        & Photo
        & Skin
        & 401,059
        & x/255
        & \href{https://challenge2024.isic-archive.com/}{$\bigstar$}
        \\
        \hline
        \textbf{Total MedVerse}
        &
        &
        & \textbf{4,973,080}
        &
    \end{tabular}
    }
\end{table*}

\subsubsection{Downstream dataset}
We carefully curate our downstream datasets for a rigorous, diverse, and clinically meaningful evaluation. The selected datasets span 10 imaging modalities, 26 tasks, and a wide range of anatomical regions, enabling us to probe the capabilities of multi-modality FMs. Our choices are guided by three key principles: \textit{1) Clinical and anatomical coverage:} We include datasets from major diagnostic categories such as cardiac, abdominal, neurological, ophthalmic, and oncological imaging to evaluate cross-anatomy transferability. \textit{2) Heterogeneous modalities and acquisition settings:} By incorporating heterogeneous modalities like CT, MRI, PET, X-ray, fundus, etc., we enable the evaluations to reflect the broad variability in imaging physics, contrast mechanisms, sensor characteristics, and clinical workflows encountered in real applications. \textit{3) Task diversity and annotation richness:} The tasks cover organ segmentation, image-level classification, and fine-grained segmentation tasks, enabling us to assess both high-level semantic understanding and low-level anatomical precision. The specific information of the downstream task is described in Tab.\ref{tab:data_down}.

\begin{table*}
    \centering
    \caption{The composition of the datasets in our downstream evaluation. It covers 26 tasks and 10 medical image modalities with varied anatomies.}
    \label{tab:data_down}
   \resizebox{\linewidth}{!}
  {
    \begin{tabular}{lccclcccccc}
        \textbf{Dataset}
        & \textbf{Modality}
        & \textbf{Anatomy}
        & \textbf{Task}
        & \textbf{Targets}
        & \textbf{Train}
        & \textbf{Val}
        & \textbf{Test}
        & \textbf{Link}
        \\
        \hline
        WHS-CT (+CAT08)$_{S}$ \cite{zhuang2019evaluation,metz2009coronary}
        & CT
        & Heart
        & Segmentation
        & 8 heart structures
        & 5,176
        & 0
        & 4,096
        & \href{https://zmiclab.github.io/zxh/0/mmwhs/}{$\bigstar$}
        \\
        KiPA22$_{S}$ \cite{he2021meta}
        & CT
        & Kidney
        & Segmentation
        & 4 kidney cancer structures
        & 13,846
        & 5,864
        & 5,959
        & \href{https://kipa22.grand-challenge.org/}{$\bigstar$}
        \\
        SISVSE$_{S}$ \cite{yoon2022surgical}
        & Endoscope
        & Stomach
        & Segmentation
        & 31 surgical scene objects
        & 2,718
        & 336
        & 1,456
        & \href{https://sisvse.github.io/}{$\bigstar$}
        \\
        CholecSeg8k$_{S}$ \cite{hong2020cholecseg8k}
        & Endoscope
        & Gall bladder
        & Segmentation
        & 13 cholecystectomy objects
        & 4,320
        & 1,120
        & 2,640
        & \href{https://www.kaggle.com/datasets/newslab/cholecseg8k}{$\bigstar$}
        \\
        IDRiD$_{S}$ (task 1) \cite{porwal2018indian}
        & Fundus
        & Eye
        & Segmentation
        & 4 lesions and optic disc
        & 54
        & 0
        & 27
        & \href{https://idrid.grand-challenge.org/}{$\bigstar$}
        \\
        FIVES$_{S}$ \cite{jin2022fives}
        & Fundus
        & Eye
        & Segmentation
        & Fundus vessels
        & 200
        & 100
        & 500
        & \href{https://figshare.com/articles/figure/FIVES_A_Fundus_Image_Dataset_for_AI-based_Vessel_Segmentation/19688169/1}{$\bigstar$}
        \\
        FeTA21$_{S}$ \cite{payette2021automatic}
        & MR
        & Brain
        & Segmentation
        & 13 feta brain tissues
        & 5,120
        & 5,120
        & 10,240
        & \href{https://feta.grand-challenge.org/feta-2021/}{$\bigstar$}
        \\
        CANDI$_{S}$ \cite{kennedy2012candishare}
        & MR
        & Brain
        & Segmentation
        & 28 brain tissues
        & 5,120
        & 2,560
        & 5,504
        & \href{https://www.nitrc.org/projects/candi_share/}{$\bigstar$}
        \\
        OCTDL$_{C}$ \cite{kulyabin2024octdl}
        & OCT
        & Eye
        & Classification
        & 7 kinds of diseases
        & 1,030
        & 411
        & 623
        & \href{https://github.com/MikhailKulyabin/OCTDL}{$\bigstar$}
        \\
        ICF$_{S}$ \cite{ahmed2022deep}
        & OCT
        & Eye
        & Segmentation
        & Cystoid Macular Edema
        & 730
        & 292
        & 438
        & \href{https://www.kaggle.com/datasets/zeeshanahmed13/intraretinal-cystoid-fluid}{$\bigstar$}
        \\
        TNBC$_{S}$ \cite{wu2021single}
        & Pathology
        & Breast
        & Segmentation
        & Cell
        & 20
        & 10
        & 20
        & \href{https://github.com/PeterJackNaylor/DRFNS?tab=readme-ov-file}{$\bigstar$}
        \\
        LSC$_{C}$ \cite{janowczyk2016deep}
        & Pathology
        & Lymphoma
        & Classification
        & 3 Lymphoma Sub-types
        & 188
        & 73
        & 113
        & \href{https://andrewjanowczyk.com/deep-learning/}{$\bigstar$}
        \\
        BCSS$_{S}$ \cite{amgad2019structured}
        & Pathology
        & Breast
        & Segmentation
        & 5 kinds of tissues or regions
        & 75
        & 30
        & 46
        & \href{https://github.com/PathologyDataScience/BCSS}{$\bigstar$}
        \\
        MoNuSAC$_{S}$ \cite{verma2021monusac2020}
        & Pathology
        & 4 kinds of organs
        & Segmentation
        & 4 kinds of nuclei
        & 167
        & 42
        & 101
        & \href{https://hf-mirror.com/datasets/RationAI/MoNuSAC}{$\bigstar$}
        \\
        SCR$_{S}$ \cite{van2006segmentation}
        & X-ray
        & Chest
        & Segmentation
        & 5 Chest structures
        & 25
        & 47
        & 175
        & \href{https://www.jsrt.or.jp/data/english/}{$\bigstar$}\href{https://www.isi.uu.nl/research/databases/}{$\bigstar$}
        \\
        PDCXR$_{C}$ \cite{kermany2018identifying}
        & X-ray
        & Chest
        & Classification
        & Pneumonia
        & 3,659
        & 1,573
        & 624
        & \href{https://data.mendeley.com/datasets/rscbjbr9sj/3}{$\bigstar$}
        \\
        DCA1$_{S}$ \cite{cervantes2019automatic}
        & X-ray
        & Coronary artery
        & Segmentation
        & Coronary artery
        & 25
        & 25
        & 84
        & \href{https://www.kaggle.com/datasets/bard2024/database-x-ray-coronary-angiograms-dca1}{$\bigstar$}
        \\
        DEX$_{S}$ \cite{liu2024dias}
        & X-ray
        & Brain
        & Segmentation
        & Brain artery
        & 30
        & 10
        & 20
        & \href{https://github.com/lseventeen/DIAS}{$\bigstar$}
        \\
        DSCA$_{S}$ \cite{zhang2025dsca}
        & X-ray
        & Brain
        & Segmentation
        & Brain artery
        & 112
        & 44
        & 68
        & \href{https://github.com/jiongzhang-john/DSCA}{$\bigstar$}
        \\
        STS-Tooth-2D$_{S}$ \cite{wang2025miccai}
        & X-ray
        & Tooth
        & Segmentation
        & Tooth
        & 100
        & 100
        & 700
        & \href{https://toothfairychallenges.github.io/miccai2023/}{$\bigstar$}
        \\
        UNS$_{S}$ \cite{KaggleUltrasoundNerve2016}
        & Ultrasound
        & Nerve
        & Segmentation
        & Nerve
        & 2,758
        & 1,078
        & 1,799
        & \href{https://www.kaggle.com/competitions/ultrasound-nerve-segmentation/overview}{$\bigstar$}
        \\
        TN5000$_{C}$ \cite{zhang2025tn5000}
        & Ultrasound
        & Thyroid nodule
        & Classification
        & Thyroid nodule
        & 3,500
        & 500
        & 1,000
        & \href{https://github.com/Qingyanyichen/TN5000-2025}{$\bigstar$}
        \\
        FHU$_{S}$ \cite{van2018automated}
        & Ultrasound
        & Fetus
        & Segmentation
        & Fetal head
        & 499
        & 200
        & 300
        &\href{https://www.kaggle.com/datasets/ankit8467/fetal-head-ultrasound-dataset-for-image-segment}{$\bigstar$}
        \\
        AutoPET-org$_{S}$ \cite{zhang2025seganypet,gatidis2022autopet}
        & PET
        & Whole body
        & Segmentation
        & 12 organs or structures
        & 18,076
        & 6,386
        & 10,465
        & \href{https://github.com/YichiZhang98/SegAnyPET}{$\bigstar$}
        \\
        SLSC$_{S}$ \cite{makhresearch_skin_lesion_2025}
        & Photo
        & Skin
        & Segmentation
        & Lesion
        & 6,675
        & 1,911
        & 964
        & \href{https://huggingface.co/datasets/makhresearch/skin-lesion-segmentation-classification}{$\bigstar$}
        \\
        SLSC$_{C}$ \cite{makhresearch_skin_lesion_2025}
        & Photo
        & Skin
        & Classification
        & 7 skin diseases
        & 6,675
        & 1,911
        & 964
        & \href{https://huggingface.co/datasets/makhresearch/skin-lesion-segmentation-classification}{$\bigstar$}  
    \end{tabular}
    }
\end{table*}
\subsection{Implementations}
Our experiments are implemented in PyTorch \cite{paszke2019pytorch}, a widely used deep learning framework. Our DEX is pretrained on two NVIDIA H100 Tensor Core GPUs with 80 GB memory, and all models are adapted on NVIDIA A100 SXUM4 GPUs with 40 GB memory. Specifically,
\subsubsection{Pretraining implementation}
During pretraining, our DEX is optimized by AdamW \cite{loshchilovdecoupled} with a base learning rate $r_{\text{base}}=10^{-4}$. The initial learning rate is $r_{\text{init}}=r_{\text{basic}}*B/512$, where batch size is $B=832$. We adopt a 10-epoch warm-up schedule to stabilize training, after which the learning rate is decayed following a cosine schedule. In our adaptive frequency-aware noise mechanism, the noise standard deviation $\sigma$ is initialized to 1 and annealed using the same cosine decay. The balance loss weight $\lambda_{\text{bal}}$ is initialized at 0.01 and is reduced using cosine decay. For data augmentation, we apply random horizontal flipping, random affine transformations, and random resized cropping, resulting in the images with $256 \times 256$ size. Then, the images are normalized via z-score with the mean of [0.485, 0.456, 0.406] and std of [0.229, 0.224, 0.225] following \cite{he2022masked}. Our DEX network adopts the ViT-B \cite{dosovitskiy2020image} as the basic architecture that applies global self-attention before our DEX module and uses the MLP layers as the experts and Director. We utilize the MAE \cite{he2022masked} as the self-supervised learning loss $\mathcal{L}_{\text{self}}$. 

\subsubsection{Downstream implementation}
Following \cite{He_2025_ICCV}, we utilize the fine-tuning evaluation to demonstrate the adaptation ability in downstream tasks. All segmentation tasks are implemented via the UNETR \cite{hatamizadeh2022unetr}, and all classification tasks adopt a classification head appended to the backbone network. The gradient optimizes all parameters through the frameworks during the training. All the downstream tasks are trained by AdamW \cite{loshchilovdecoupled} optimizer with the learning rate of $10^{-4}$. We adopt a 10-epoch warm-up schedule to stabilize training, after which the learning rate is decayed following a cosine schedule. For segmentation tasks, we use the sum of Dice loss \cite{zhao2020rethinking} and cross-entropy loss \cite{mao2023cross} between the predicted masks and ground truths to learn segmentation. For classification tasks, we utilize the cross-entropy loss \cite{mao2023cross} between the predicted categories and ground truths to train the classification.

\section{More Framework Analysis and Results}
\begin{table}[htb]
    \centering
    \caption{The comparison of the adaptation via experts (DEX-$\mathcal{E}$) and Director (DEX-$\mathcal{D}$). The DEX-$\mathcal{E}$ has great specialization ability, bringing better transferring ability.}
  \resizebox{\linewidth}{!}
  {
    \begin{tabular}{l|cc|ccccccc}
        \multirow{2}{*}{\textbf{Task}}
        & DEX-$\mathcal{E}$
        & DEX-$\mathcal{D}$
        & DEX-$\mathcal{E}$
        & DEX-$\mathcal{D}$
        \\
        & \textit{scratch}
        & \textit{scratch}
        & MedVerse
        & MedVerse
        \\
        \hline
        \textbf{CT}-WHS$_{S}$
        & 77.9
        & 77.8
        & 88.9$_{+11.0}$
        & 82.6$_{+4.8}$
        \\ 
        \textbf{Fundus}-IDRiD$_{S}$
        & 33.7
        & 42.0
        & 61.2$_{+27.5}$
        & 60.8$_{+18.8}$
        \\ 
        \textbf{Path}-TNBC$_{S}$
        & 73.1
        & 72.2
        & 81.1$_{8.0}$
        & 77.3$_{8.1}$
        \\ 
        \textbf{X-ray}-PDCXR$_{C}$
        & 82.7
        & 83.6
        & 90.7$_{+8.0}$
        & 85.6$_{+2.0}$
        \\ 
        \textbf{Photo}-SISC$_{C}$
        & 54.4
        & 55.4
        & 80.6$_{26.2}$
        & 58.8$_{3.4}$
    \end{tabular}
    }
    \label{tab:AD}
\end{table}

\subsection{Adaptation via Experts \textit{\textbf{v.s.}} Director}
As shown in Tab. \ref{tab:AD}, we compare the downstream adaptation of DEX when fine-tuning either the experts (DEX-$\mathcal{E}$, $K$=2) or the Director (DEX-$\mathcal{D}$, which degenerates into a ViT-base structure), aiming to understand how pretrained modular roles contribute to task transfer. Three consistent observations emerge. 1) When trained from scratch, the two variants achieve comparable performance, indicating that without pretraining, neither the experts nor the director alone provides an inherent advantage. 2) After MedVerse pretraining, DEX-$\mathcal{E}$ yields substantially larger improvements across almost all tasks, for example, +11.0 on CT-WHS. This evidence suggests that experts acquire highly specialized, modality-aware knowledge, making their parameter subspaces more directly transferable to new tasks. 3) DEX-$\mathcal{D}$ consistently improves over scratch but lags behind DEX-$\mathcal{E}$, reflecting the director's role as a modality aggregator that regulates high-level alignment rather than encoding detailed intra-modality structures. Collectively, these results confirm our design intuition: the experts internalize fine-grained modality-specific semantics during pretraining, while the director primarily learns to mediate alignment. Thus, adapting the experts unleashes the specialized representations accumulated during pretraining, leading to superior downstream performance across heterogeneous medical tasks.

\begin{table*}
  \centering
  \caption{A total of 26 publicly available tasks are involved in our experiment for wide evaluations. The ``S'' and ``C'' are the segmentation and classification tasks.}\label{tab:comp_all}
  \resizebox{\linewidth}{!}
  {
  \begin{tabular}{l|c|ccccccccccc|ccccccccccccc}
    \textbf{Model}
    & \color{gray}-
    & MAE
    & DINOv2
    & DINOv3
    & PLIP
    & CONCH
    & USFM
    & RETFound
    & PanDerm
    & LVMMed
    & MedSAM
    & MAE-MedVerse
    & \textbf{DEX}
    \\ 
    \textbf{Modality}
    & \color{gray}\textit{scratch}
    & NI
    & NI
    & NI
    & Path
    & Path
    & US
    & OCT
    & MM
    & MM
    & MM
    & MM
    & MM
    \\ 
    \textbf{Backbone}
    & \color{gray}ViT-B
    & ViT-B
    & ViT-B
    & ViT-B
    & ViT-B
    & ViT-B
    & ViT-B
    & ViT-L
    & ViT-B
    & ViT-B
    & ViT-B
    & ViT-B
    & ViT-B
    \\
    \textbf{Data size}
    & \color{gray}-
    & 1.28M
    & 120M
    & 1.7B
    & 208K
    & 1.17M
    & 2.19M
    & 736K
    & 2.15M
    & 1.3M
    & 1.57M
    & 4.97M
    & 4.97M
    \\
    \hline
    \textbf{CT}-WHS$_{S}$ \cite{zhuang2019evaluation,metz2009coronary}
    & \color{gray}77.8
    & 84.4$_{+6.6}$
    & \cellcolor[HTML]{E9DAF3}\textbf{88.9$_{+11.1}$}
    & \cellcolor[HTML]{C4D2E7}88.7$_{+10.9}$
    & 86.2$_{+8.4}$
    & 86.9$_{+9.1}$
    & 83.9$_{+6.1}$
    & 81.2$_{+3.4}$
    & \cellcolor[HTML]{C4D2E7}88.7$_{+10.9}$
    & 85.2$_{+7.4}$
    & 82.7$_{+4.9}$
    & 84.3$_{+6.5}$
    & \cellcolor[HTML]{E9DAF3}\textbf{88.9$_{+11.1}$}
    \\
    \textbf{CT}-KiPA22$_{S}$ \cite{he2021meta}
    & \color{gray}68.4
    & 77.1$_{+8.7}$
    & 77.9$_{+9.5}$
    & \cellcolor[HTML]{E9DAF3}\textbf{79.7$_{+11.3}$}
    & 72.3$_{+3.9}$
    & 78.0$_{+9.6}$
    & 70.9$_{+2.5}$
    & 74.6$_{+6.2}$
    & 77.5$_{+9.1}$
    & \cellcolor[HTML]{C4D2E7}78.3$_{+9.9}$
    & 73.6$_{+5.2}$
    & 76.5$_{+8.1}$
    & \cellcolor[HTML]{C4D2E7}78.9$_{+10.5}$
    \\
    \textbf{Endo}-SISVSE$_{S}$ \cite{yoon2022surgical}
    & \color{gray}53.3
    & 69.0$_{+15.7}$
    & 64.5$_{+11.2}$
    & \cellcolor[HTML]{E9DAF3}\textbf{74.5$_{+21.2}$}
    & 58.9$_{+5.6}$
    & 68.3$_{+15.0}$
    & 65.1$_{+11.8}$
    & 66.1$_{+12.8}$
    & 68.3$_{+15.0}$
    & 70.5$_{+17.2}$
    & \cellcolor[HTML]{C4D2E7}71.6$_{+18.3}$
    & 65.0$_{+11.7}$
    & \cellcolor[HTML]{C4D2E7}70.9$_{+17.6}$
    \\
    \textbf{Endo}-CholecSeg8k$_{S}$ \cite{hong2020cholecseg8k}
    & \color{gray}44.7
    & 55.4$_{+10.7}$
    & 57.2$_{+12.5}$
    & \cellcolor[HTML]{E9DAF3}\textbf{61.2$_{+16.5}$}
    & 53.0$_{+8.3}$
    & 56.3$_{+11.6}$
    & 53.9$_{+9.2}$
    & 53.0$_{+8.3}$
    & 55.6$_{+10.9}$
    & \cellcolor[HTML]{C4D2E7}58.4$_{+13.7}$
    & 56.9$_{+12.2}$
    & 54.4$_{+9.7}$
    & \cellcolor[HTML]{C4D2E7}59.1$_{+14.4}$
    \\
    \textbf{Fundus}-IDRiD$_{S}$ \cite{porwal2018indian}
    & \color{gray}42.0
    & 54.1$_{+12.1}$
    & 44.0$_{+2.0}$
    & 36.2$_{-5.8}$
    & 31.8$_{-10.2}$
    & 52.5$_{+10.5}$
    & 49.3$_{+7.3}$
    & 54.3$_{+12.3}$
    & 51.2$_{+9.2}$
    & 53.7$_{+11.7}$
    & \cellcolor[HTML]{C4D2E7}56.8$_{+14.8}$
    & \cellcolor[HTML]{C4D2E7}60.2$_{+18.2}$
    & \cellcolor[HTML]{E9DAF3}\textbf{61.2$_{+19.2}$}
    \\
    \textbf{Fundus}-FIVES$_{S}$ \cite{jin2022fives}
    & \color{gray}84.4
    & 89.2$_{+4.8}$
    & 87.1$_{+2.7}$
    & \cellcolor[HTML]{C4D2E7}89.5$_{+5.1}$
    & 88.3$_{+3.9}$
    & 88.3$_{+3.9}$
    & 86.3$_{+1.9}$
    & 87.4$_{+3.0}$
    & 89.3$_{+4.9}$
    & \cellcolor[HTML]{C4D2E7}89.5$_{+5.1}$
    & 86.6$_{+2.2}$
    & 87.1$_{+2.7}$
    & \cellcolor[HTML]{E9DAF3}\textbf{90.3$_{+5.9}$}
    \\
    \textbf{MR}-FeTA21$_{S}$ \cite{payette2021automatic}
    & \color{gray}35.2
    & \cellcolor[HTML]{E9DAF3}\textbf{56.5$_{+21.3}$}
    & 55.9$_{+20.7}$
    & \cellcolor[HTML]{C4D2E7}56.2$_{+21.0}$
    & 52.8$_{+17.6}$
    & 51.0$_{+15.8}$
    & 46.2$_{+11.0}$
    & 50.6$_{+15.4}$
    & 54.9$_{+19.7}$
    & 47.8$_{+12.6}$
    & 41.6$_{+6.4}$
    & 41.1$_{+5.9}$
    & \cellcolor[HTML]{E9DAF3}\textbf{56.5$_{+21.3}$}
    \\
    \textbf{MR}-CANDI$_{S}$ \cite{kennedy2012candishare}
    & \color{gray}88.9
    & \cellcolor[HTML]{C4D2E7}89.4$_{+0.5}$
    & 89.2$_{+0.3}$
    & 89.3$_{+0.4}$
    & 89.1$_{+0.2}$
    & 89.2$_{+0.3}$
    & 89.1$_{+0.2}$
    & 89.0$_{+0.1}$
    & \cellcolor[HTML]{C4D2E7}89.4$_{+0.5}$
    & \cellcolor[HTML]{E9DAF3}\textbf{89.5$_{+0.6}$}
    & 89.1$_{+0.2}$
    & 89.0$_{+0.1}$
    & \cellcolor[HTML]{C4D2E7}89.4$_{+0.5}$
    \\
    \textbf{OCT}-OCTDL$_{C}$ \cite{kulyabin2024octdl}
    & \color{gray}11.5
    & 66.2$_{+54.7}$
    & 75.9$_{+64.4}$
    & \cellcolor[HTML]{C4D2E7}83.1$_{+71.6}$
    & 52.8$_{+41.3}$
    & 74.3$_{+62.8}$
    & 61.5$_{+50.0}$
    & \cellcolor[HTML]{C4D2E7}88.2$_{+76.7}$
    & 75.8$_{+64.3}$
    & 75.3$_{+63.8}$
    & 48.9$_{+37.4}$
    & 72.5$_{+61.0}$
    & \cellcolor[HTML]{E9DAF3}\textbf{88.8$_{+77.3}$}
    \\
    \textbf{OCT}-ICF$_{S}$ \cite{ahmed2022deep}
    & \color{gray}83.7
    & 86.7$_{+3.0}$
    & 85.8$_{+2.1}$
    & \cellcolor[HTML]{C4D2E7}86.9$_{+3.2}$
    & 84.4$_{+0.7}$
    & 86.1$_{+2.4}$
    & 85.4$_{+1.7}$
    & \cellcolor[HTML]{E9DAF3}\textbf{87.0$_{+3.3}$}
    & 86.3$_{+2.6}$
    & 86.4$_{+2.7}$
    & 86.7$_{+3.0}$
    & \cellcolor[HTML]{C4D2E7}86.8$_{+3.1}$
    & 86.2$_{+2.5}$
    \\
    \textbf{Path}-LSC$_{C}$ \cite{janowczyk2016deep}
    & \color{gray}31.5
    & 52.2$_{+20.7}$
    & 53.3$_{+21.8}$
    & 54.4$_{+22.9}$
    & 55.7$_{+24.2}$
    & \cellcolor[HTML]{C4D2E7}66.6$_{+35.1}$
    & 59.5$_{+28.0}$
    & \cellcolor[HTML]{C4D2E7}60.7$_{+29.2}$
    & 50.4$_{+18.9}$
    & 54.6$_{+23.1}$
    & 56.8$_{+25.3}$
    & 51.3$_{+19.8}$
    & \cellcolor[HTML]{E9DAF3}\textbf{73.2$_{+41.7}$}
    \\
    \textbf{Path}-TNBC$_{S}$ \cite{wu2021single}
    & \color{gray}72.2
    & \cellcolor[HTML]{C4D2E7}81.3$_{+9.1}$
    & 75.4$_{+3.2}$
    & 74.9$_{+2.7}$
    & 74.1$_{+1.9}$
    & \cellcolor[HTML]{C4D2E7}81.8$_{+9.6}$
    & 79.5$_{+7.3}$
    & 68.0$_{-4.2}$
    & 81.2$_{+9.0}$
    & \cellcolor[HTML]{E9DAF3}\textbf{82.5$_{+10.3}$}
    & 78.0$_{+5.8}$
    & 78.4$_{+6.2}$
    & 81.1$_{+8.9}$
    \\
    \textbf{Path}-BCSS$_{S}$ \cite{amgad2019structured}
    & \color{gray}57.8
    & 68.3$_{+10.5}$
    & 62.6$_{+4.8}$
    & \cellcolor[HTML]{E9DAF3}\textbf{71.6$_{+13.8}$}
    & 58.7$_{+0.9}$
    & \cellcolor[HTML]{E9DAF3}\textbf{71.6$_{+13.8}$}
    & 62.8$_{+5.0}$
    & 63.4$_{+5.6}$
    & 69.0$_{+11.2}$
    & 70.2$_{+12.4}$
    & 64.9$_{+7.1}$
    & 66.3$_{+8.5}$
    & \cellcolor[HTML]{E9DAF3}\textbf{71.6$_{+13.8}$}
    \\
    \textbf{Path}-MoNuSAC$_{S}$ \cite{verma2021monusac2020}
    & \color{gray}26.9
    & 42.8$_{+15.9}$
    & 42.5$_{+15.6}$
    & 45.0$_{+18.1}$
    & 35.6$_{+8.7}$
    & \cellcolor[HTML]{C4D2E7}51.7$_{+24.8}$
    & 44.9$_{+18.0}$
    & 35.8$_{+8.9}$
    & 46.7$_{+19.8}$
    & \cellcolor[HTML]{E9DAF3}\textbf{56.7$_{+29.8}$}
    & 46.2$_{+19.3}$
    & 47.6$_{+20.7}$
    & \cellcolor[HTML]{C4D2E7}53.6$_{+26.7}$
    \\
    \textbf{X-ray}-PDCXR$_{C}$ \cite{kermany2018identifying}
    & \color{gray}83.6
    & 82.7$_{-0.9}$
    & 86.3$_{+2.7}$
    & \cellcolor[HTML]{C4D2E7}88.6$_{+5.0}$
    & 84.5$_{+0.9}$
    & \cellcolor[HTML]{C4D2E7}87.1$_{+3.5}$
    & 84.3$_{+0.7}$
    & 84.0$_{+0.4}$
    & \cellcolor[HTML]{C4D2E7}87.1$_{+3.5}$
    & 86.6$_{+3.0}$
    & 85.3$_{+1.7}$
    & 86.1$_{+2.5}$
    & \cellcolor[HTML]{E9DAF3}\textbf{90.7$_{+7.1}$}
    \\
    \textbf{X-ray}-SCR$_{S}$ \cite{van2006segmentation}
    & \color{gray}61.6
    & 94.0$_{+32.4}$
    & 92.0$_{+30.4}$
    & 93.0$_{+31.4}$
    & 91.7$_{+30.1}$
    & 93.5$_{+31.9}$
    & 92.3$_{+30.7}$
    & 92.9$_{+31.3}$
    & \cellcolor[HTML]{C4D2E7}94.4$_{+32.8}$
    & \cellcolor[HTML]{C4D2E7}94.4$_{+32.8}$
    & 93.1$_{+31.5}$
    & 92.8$_{+31.2}$
    & \cellcolor[HTML]{E9DAF3}\textbf{94.9$_{+33.3}$}
    \\
    \textbf{X-ray}-DCA1$_{S}$ \cite{cervantes2019automatic}
    & \color{gray}70.6
    & 75.6$_{+5.0}$
    & 75.2$_{+4.6}$
    & 74.2$_{+3.6}$
    & 73.5$_{+2.9}$
    & 75.6$_{+5.0}$
    & 74.0$_{+3.4}$
    & 74.6$_{+4.0}$
    & 75.7$_{+5.1}$
    & \cellcolor[HTML]{C4D2E7}76.4$_{+5.8}$
    & \cellcolor[HTML]{C4D2E7}75.8$_{+5.2}$
    & 74.6$_{+4.0}$
    & \cellcolor[HTML]{E9DAF3}\textbf{76.8$_{+6.2}$}
    \\
    \textbf{X-ray}-DEX$_{S}$ \cite{liu2024dias}
    & \color{gray}71.1
    & \cellcolor[HTML]{C4D2E7}73.9$_{+2.8}$
    & 73.6$_{+2.6}$
    & \cellcolor[HTML]{C4D2E7}73.9$_{+2.8}$
    & 72.7$_{+1.6}$
    & 73.1$_{+2.0}$
    & 72.7$_{+1.6}$
    & 73.3$_{+2.2}$
    & \cellcolor[HTML]{C4D2E7}73.9$_{+2.8}$
    & 73.8$_{+2.7}$
    & 72.2$_{+1.1}$
    & 73.3$_{+2.2}$
    & \cellcolor[HTML]{E9DAF3}\textbf{74.3$_{+3.2}$}
    \\
    \textbf{X-ray}-DSCA$_{S}$ \cite{zhang2025dsca}
    & \color{gray}82.6
    & 84.9$_{+2.3}$
    & 85.0$_{+2.4}$
    & \cellcolor[HTML]{E9DAF3}\textbf{85.8$_{+3.2}$}
    & 84.2$_{+1.6}$
    & 84.7$_{+2.1}$
    & 84.3$_{+1.7}$
    & 84.4$_{+1.8}$
    & 85.2$_{+2.6}$
    & \cellcolor[HTML]{C4D2E7}85.4$_{+2.8}$
    & 84.0$_{+1.4}$
    & 85.3$_{+2.7}$
    & \cellcolor[HTML]{E9DAF3}\textbf{85.8$_{+3.2}$}
    \\
    \textbf{X-ray}-STS-Tooth-2D$_{S}$ \cite{wang2025miccai}
    & \color{gray}90.1
    & 92.8$_{+2.7}$
    & 92.6$_{+2.5}$
    & \cellcolor[HTML]{E9DAF3}\textbf{92.9$_{+2.8}$}
    & 92.3$_{+2.2}$
    & 92.6$_{+2.5}$
    & 92.2$_{+2.1}$
    & 92.2$_{+2.1}$
    & \cellcolor[HTML]{E9DAF3}\textbf{92.9$_{+2.8}$}
    & \cellcolor[HTML]{E9DAF3}\textbf{92.9$_{+2.8}$}
    & 92.4$_{+2.3}$
    & 92.3$_{+2.2}$
    & 92.7$_{+2.6}$
    \\
    \textbf{US}-TN5000$_{C}$ \cite{zhang2025tn5000}
    & \color{gray}2.8
    & 77.9$_{+75.1}$
    & 76.2$_{+73.4}$
    & \cellcolor[HTML]{C4D2E7}79.5$_{+76.7}$
    & 18.3$_{+15.5}$
    & 77.8$_{+75.0}$
    & 64.9$_{+62.1}$
    & 72.4$_{+69.6}$
    & 73.1$_{+70.3}$
    & 77.3$_{+74.5}$
    & \cellcolor[HTML]{C4D2E7}79.3$_{+76.5}$
    & 78.0$_{+75.2}$
    & \cellcolor[HTML]{E9DAF3}\textbf{83.5$_{+80.7}$}
    \\
    \textbf{US}-UNS$_{S}$ \cite{KaggleUltrasoundNerve2016}
    & \color{gray}58.4
    & 68.4$_{+10.0}$
    & 61.5$_{+3.1}$
    & 69.3$_{+10.9}$
    & 61.2$_{+2.8}$
    & 63.3$_{+4.9}$
    & 59.7$_{+1.3}$
    & 63.4$_{+5.0}$
    & 65.4$_{+7.0}$
    & \cellcolor[HTML]{C4D2E7}70.3$_{+11.9}$
    & \cellcolor[HTML]{E9DAF3}\textbf{71.1$_{+12.7}$}
    & 64.7$_{+6.3}$
    & \cellcolor[HTML]{C4D2E7}70.2$_{+11.8}$
    \\
    \textbf{US}-FHU$_{S}$ \cite{van2018automated}
    & \color{gray}81.9
    & 92.8$_{+10.9}$
    & 91.8$_{+9.9}$
    & \cellcolor[HTML]{E9DAF3}\textbf{93.8$_{+11.9}$}
    & 88.7$_{+6.8}$
    & 91.5$_{+9.6}$
    & 85.8$_{+3.9}$
    & 89.3$_{+7.4}$
    & \cellcolor[HTML]{C4D2E7}93.2$_{+11.3}$
    & 92.8$_{+10.9}$
    & 92.8$_{+10.9}$
    & 88.8$_{+6.9}$
    & \cellcolor[HTML]{E9DAF3}\textbf{93.8$_{+11.9}$}
    \\
    \textbf{PET}-AutoPET-org$_{S}$ \cite{zhang2025seganypet,gatidis2022autopet}
    & \color{gray}7.6
    & 19.5$_{+11.9}$
    & 27.2$_{+19.6}$
    & 35.8$_{+28.2}$
    & 4.7$_{-2.9}$
    & 13.0$_{+5.4}$
    & 15.1$_{+7.5}$
    & \cellcolor[HTML]{C4D2E7}62.6$_{+55.0}$
    & \cellcolor[HTML]{E9DAF3}\textbf{65.0$_{+57.4}$}
    & 42.4$_{+34.8}$
    & \cellcolor[HTML]{C4D2E7}55.6$_{+48.0}$
    & 35.8$_{+28.2}$
    & 52.1$_{+44.4}$
    \\
    \textbf{Photo}-SLSC$_{S}$ \cite{makhresearch_skin_lesion_2025}
    & \color{gray}90.2
    & 93.1$_{+2.9}$
    & 92.7$_{+2.5}$
    & 93.4$_{+3.2}$
    & 91.2$_{+1.0}$
    & 93.0$_{+2.8}$
    & 92.1$_{+1.9}$
    & 91.9$_{+1.7}$
    & \cellcolor[HTML]{C4D2E7}93.5$_{+3.3}$
    & 92.9$_{+2.7}$
    & \cellcolor[HTML]{E9DAF3}\textbf{93.6$_{+3.4}$}
    & 91.3$_{+1.1}$
    & \cellcolor[HTML]{E9DAF3}\textbf{93.6$_{+3.4}$}
    \\
    \textbf{Photo}-SLSC$_{C}$ \cite{makhresearch_skin_lesion_2025}
    & \color{gray}55.4
    & 77.8$_{+22.4}$
    & 74.6$_{+19.2}$
    & \cellcolor[HTML]{E9DAF3}\textbf{82.8$_{+27.4}$}
    & 73.4$_{+18.0}$
    & \cellcolor[HTML]{C4D2E7}80.7$_{+25.3}$
    & 67.9$_{+12.5}$
    & 73.0$_{+17.6}$
    & 77.3$_{+21.9}$
    & 76.9$_{+21.5}$
    & 74.3$_{+18.9}$
    & 77.2$_{+21.8}$
    & \cellcolor[HTML]{C4D2E7}80.6$_{+25.2}$
    \\
    \hline
    \textbf{Avg.} (26 tasks)
    & \color{gray}59.0
    & 73.3$_{+14.3}$
    & 72.7$_{+13.7}$
    & 75.2$_{+16.2}$
    & 66.5$_{+7.5}$
    & 74.2$_{+15.2}$
    & 70.1$_{+11.1}$
    & 73.6$_{+14.6}$
    & \cellcolor[HTML]{C4D2E7}75.4$_{+16.4}$
    & \cellcolor[HTML]{C4D2E7}75.4$_{+16.4}$
    & 73.5$_{+14.5}$
    & 73.1$_{+14.1}$
    & \cellcolor[HTML]{E9DAF3}\textbf{78.4$_{+19.4}$}
  \end{tabular}
  }
\end{table*}

\subsection{Comparison Results on All Tasks}
As shown in Tab.\ref{tab:comp_all}, we present the comprehensive comparison across 26 diverse medical vision tasks, and DEX achieves the strongest overall results, obtaining the highest average score of 78.4 and consistently outperforming both large-scale general vision FMs (DINOv2, DINOv3) and uni-modality medical FMs (RETFound, PLIP, CONCH, MedSAM, LVMMed). Despite being trained on significantly fewer images than billion-scale general vision FMs, DEX demonstrates superior transferability, especially on challenging medical modalities such as PET, ultrasound, pathology, and OCT, where it yields clear performance gains even over domain-specific pretrained models. Notably, DEX excels in both low-contrast and high-noise scenarios, such as PET AutoPET-org and US TN5000, while also delivering top-tier performance on high-contrast tasks, including CT, MR, fundus, and X-ray. These consistent improvements across classification, segmentation, and instance segmentation tasks indicate that DEX learns modality-robust and semantically disentangled representations that generalize more effectively than existing medical or general-domain pretraining approaches.
\subsection{Attention Maps on More Modalities}
\begin{figure*}[htb]
  \centering
  \includegraphics[width=0.9\linewidth]{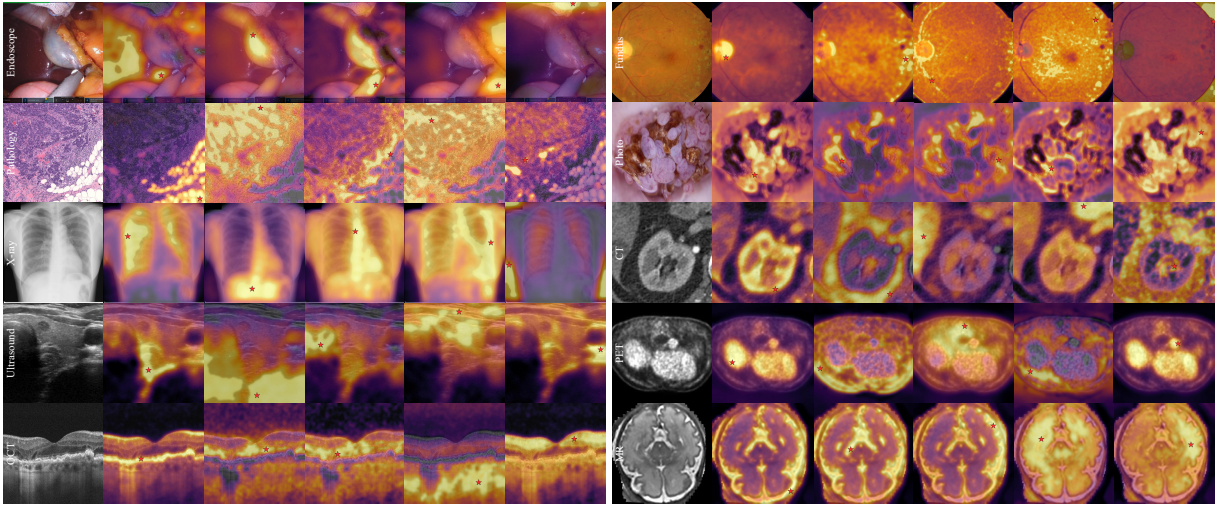}
  \caption{Extended visualization of attention maps from DEX. Each row denotes a modality, and each column shows the attention of a query patch ({\color{red}$\bigstar$}). DEX focuses on semantically meaningful regions, showing strong semantic decoupling ability of internal images.}
  \label{fig:heatmap_all}
\end{figure*}
As shown in Fig.\ref{fig:heatmap_all}, we have visualized more attention maps on different modalities. Across all modalities, our DEX attends to semantically coherent and structurally meaningful regions, indicating that the pretraining successfully encourages strong semantic decoupling. For high-contrast modalities such as endoscopy, histology, and X-ray, the query patch ({\color{red}$\bigstar$}) reliably focuses on edges, anatomical boundaries, or tissue structures. This demonstrates that DEX effectively captures the underlying semantic subspaces that are relevant for downstream tasks, thereby supporting its strong transferability. In contrast, for low-contrast or high-noise modalities, like the PET, OCT, and ultrasound, the attention maps become more diffuse. Due to minimal gray-level variation and weak local texture cues, the pretrained representations struggle to further discriminate between visually similar regions, resulting in less sharply defined attention boundaries. Although semantic consistency is still preserved to some extent, the decoupling ability is naturally limited by the inherent information content of these modalities. Overall, these visualizations confirm that DEX learns the common representation across modalities and provide intuitive evidence for its superior downstream performance. Its pretrained semantic discrimination enables the model to adapt and identify regions across heterogeneous medical image modalities.

\subsection{Pattern Layout on More Modalities}
\begin{figure}[htb]
  \centering
  \includegraphics[width=\linewidth]{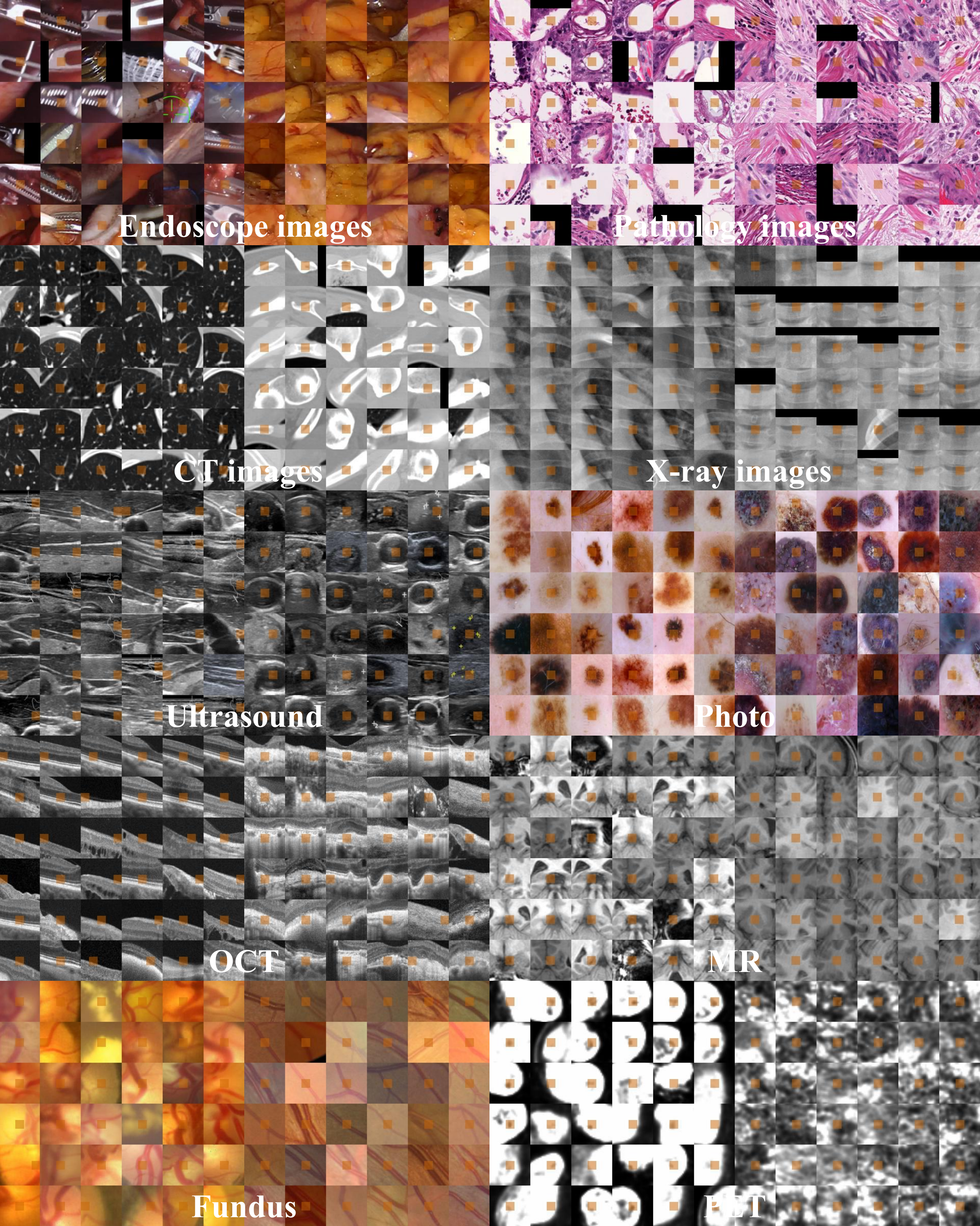}
  \caption{Extended pattern layout learned by DEX across 10 imaging modalities. For each modality, we visualize clusters of image patches that share high feature similarity in the learned representation space. DEX consistently discovers semantically coherent regions, such as instruments, fat, lumen, stroma, lung textures, bone structures, retinal layers, fundus vessels, and PET metabolic patterns, demonstrating its strong intra-modal discriminability and its ability to organize local visual patterns into meaningful semantic structures.}\label{fig:layout_all}
\end{figure}
As shown in Fig.\ref{fig:layout_all}, DEX learns clear pattern layouts across ten heterogeneous medical imaging modalities, including endoscopy, pathology, CT, X-ray, ultrasound, dermatology photos, OCT, MR, fundus, and PET. For each modality, we visualize representative feature clusters from the learned embedding space. The results reveal DEX's strong intra-modal structure-discovery ability: it consistently groups patches with similar structural, textural, or anatomical properties into coherent semantic clusters. Examples include instruments, fat, and mucosa in endoscopy; lumen and stroma in pathology; lung, bone, and soft-tissue patterns in CT and X-ray; echogenic textures in ultrasound; retinal layers in OCT; vascular and lesion patterns in fundus and photo images; and metabolic uptake regions in PET. Overall, the emergent pattern layout demonstrates that DEX not only encodes discriminative features but also uncovers modality-specific semantic regularities, forming structured and interpretable intra-modal manifolds that underlie its strong performance across diverse imaging tasks.
\subsection{Component Ablation on More Modalities}
\begin{table}[htb]
    \centering
    \caption{Component ablation across 10 representative modalities. Starting from self-supervised pretraining on MedVerse, we progressively introduce the expert and director modules in DEX. Results consistently demonstrate that both specialization and coordination contribute positively across heterogeneous modalities under the Non-IID setting.}
    \resizebox{\linewidth}{!}
    {
    \begin{tabular}{cccccc}
        Modality
        & Task
        & \textit{scratch}
        & $\mathcal{L}_{\text{self}}$ (MedVerse)
        & + Experts
        & + Director
        \\
        \hline
        X-ray
        & PDCXR$_{C}$
        & 83.6
        & 86.1$_{+2.5}$
        & 88.4$_{+4.8}$
        & 90.7$_{+7.1}$
        \\
        Path
        & TNBC$_{S}$
        & 72.2
        & 78.4$_{+6.2}$
        & 78.5$_{+6.3}$
        & 81.1$_{+8.9}$
        \\
        CT
        & WHS$_{S}$
        & 77.8
        & 84.3$_{+6.5}$
        & 86.0$_{+8.2}$
        & 88.9$_{+11.1}$
        \\
        Endo
        & SISVSE$_{S}$
        & 53.3
        & 65.0$_{+11.7}$
        & 67.7$_{+14.4}$
        & 70.9$_{+17.6}$
        \\
        Fundus
        & FIVES$_{S}$
        & 84.4
        & 87.1$_{+2.7}$
        & 87.5$_{+3.1}$
        & 90.3$_{+5.9}$
        \\
        MR
        & FeTA21$_{S}$
        & 35.2
        & 41.1$_{+5.9}$
        & 50.0$_{+14.8}$
        & 56.5$_{+21.3}$
        \\
        OCT
        & OCTDL$_{C}$
        & 11.5
        & 72.5$_{+67}$
        & 77.8$_{+66.3}$
        & 88.8$_{+77.3}$
        \\
        US
        & FHU$_{S}$
        & 81.9
        & 88.8$_{+6.9}$
        & 92.5$_{+10.6}$
        & 93.8$_{+11.9}$
        \\
        PET
        & AutoPET-org$_{S}$
        & 7.6
        & 35.8$_{+28.2}$
        & 49.2$_{+41.6}$
        & 52.1$_{+44.5}$
        \\
        Photo
        & $SISC_{C}$
        & 90.2
        & 91.3$_{+1.1}$
        & 92.3$_{+2.1}$
        & 93.6$_{+3.4}$
        \\
        \hline
        -
        & Avg
        & 59.8
        & 73.0$_{+13.3}$
        & 77.0$_{+17.2}$
        & 80.7$_{+20.9}$
    \end{tabular}
    }
    \label{tab:abla_supp}
\end{table}

Tab.~\ref{tab:abla_supp} presents component ablations across 10 representative modalities. Self-supervised pretraining on MedVerse consistently improves performance over training from scratch across all modalities, increasing the average performance from 59.8 to 73.0, demonstrating the effectiveness of large-scale multi-modality pretraining. Introducing the expert modules further improves the average performance to 77.0, with particularly notable gains in modalities with larger distribution gaps such as MR, PET, and OCT, suggesting that expert specialization helps alleviate the Non-IID challenge across heterogeneous modalities. Adding the director module further improves the average performance to 80.7, indicating that coordinating knowledge across experts is also critical for transferable representation learning. Overall, the consistent improvements across diverse modalities validate the effectiveness of emergent modular representations in multi-modality MVFMs.
\end{document}